\RequirePackage{fix-cm}
\documentclass[smallextended]{svjour3}       
\smartqed  

\usepackage{amsmath}
\usepackage{graphicx}
\usepackage[utf8]{inputenc}
\usepackage{amsfonts}
\usepackage{comment}
\usepackage{footnote}
\usepackage{appendix}
\usepackage{natbib}
\usepackage{amsthm}
\usepackage{amssymb}

\usepackage{graphicx}
\usepackage{caption}
\usepackage{subfigure}
\usepackage{makecell}
\usepackage{adjustbox}
\usepackage{float}
\usepackage{multirow}
\usepackage{bbding}
\usepackage{pifont}
\usepackage{wasysym}

\usepackage{hyperref}

\usepackage[ruled,linesnumbered]{algorithm2e}
\newtheorem{prop}{Proposition}

\usepackage{mathrsfs} 
\begin{document}

\title{A Robust Multi-Objective Bayesian Optimization Framework Considering Input Uncertainty
}


\author{Jixiang Qing         \and
        Ivo Couckuyt \and
        Tom Dhaene 
}


\institute{J. Qing \at
              Ghent University -- imec, IDLab, Department of Information Technology (INTEC), Tech Lane -- Zwijnaarde 126, 9052 Ghent, Belgium \\
              \email{Jixiang.Qing@UGent.be}           
}
\date{Received: date / Accepted: date}

\maketitle


\section*{ABSTRACT}
Bayesian optimization is a popular tool for data-efficient optimization of expensive objective functions. In real-life applications like engineering design, the designer often wants to take multiple objectives as well as input uncertainty into account to find a set of robust solutions. While this is an active topic in single-objective Bayesian optimization, it is less investigated in the multi-objective case. We introduce a novel Bayesian optimization framework to efficiently perform multi-objective optimization considering input uncertainty. We propose a robust Gaussian Process model to infer the Bayes risk criterion to quantify robustness, and we develop a two-stage Bayesian optimization process to search for a robust Pareto frontier. The complete framework supports various distributions of the input uncertainty and takes full advantage of parallel computing. We demonstrate the effectiveness of the framework through numerical benchmarks.

\section{Introduction}
In many real-life applications, we are faced with multiple conflicting goals. For instance, tuning the topology of neural networks for accuracy as well as inference time \citep{fernandez2020improved}. A solution that is optimal for all objectives usually does not exist, and one has to compromise: identify a set of solutions that provides a trade-off among different objectives. Moreover, the calculation of the objectives sometimes requires a significant computational effort. Hence, a Multi-Objective Optimization (MOO) strategy, which is able to quickly and efficiently locate all the optimal trade-offs, is of practical interest.

Multi-Objective Bayesian Optimization (MOBO) (e.g., \citep{daulton2020differentiable,yang2019efficient}) is a well-established efficient global optimization technique to search for an optimal trade-off between conflicting objectives. Its useful properties, including data-efficiency and an agnostic treatment of the objective function, have made MOBO a widely applicable optimization technique, especially where the objectives are  time-consuming to evaluate.

In a chaotic world full of uncertainties, it is almost impossible to implement an optimal solution exactly as defined. For instance, consider an optimal configuration of a system found by MOBO. Any manufacturing uncertainty could result in a slightly different configuration and hence result in a possible degradation of the actual performance. Among these uncertainties, we are specifically interested in considering \textbf{input uncertainty}: a common uncertainty type caused by perturbations of the input parameters, that might result in different outputs. Considering input uncertainty in MOO is important to ensure that the final implemented optimal solutions are still likely to be satisfactory. Hence, it is also of high interest in MOBO.

\textbf{Limitation of current approaches}
Data-efficient approaches have been proposed to perform MOBO considering input uncertainty (\citet{zhou2018multi, rivier2018surrogate}). These approaches extend existing robust MOO methodologies with a computationally efficient surrogate model, however, the surrogate model is only utilized in a non-Bayesian way, i.e, the posterior mean is used as a point estimation, and the model refinement step has to be defined explicitly. This has usually resulted a complicated robust MOO framework. Motivated by these, we propose a lightweight robust MOO framework that deals with robustness in a principle way and still enjoys the elegance of the standard BO flow.


\textbf{Contributions}
This paper introduces a Robust Multi-Objective Bayesian Optimization framework to pursue a set of optimal solutions that considers Input Uncertainty  (RMOBO-IU). In order to handle the input uncertainty, we optimize a robust objective function, defined as the mean of the objective distribution induced by the input uncertainty, also known as Bayes risk \citep{beland2017bayesian} (see Fig. \ref{fig:idea_demo}) \citep{deb2005searching}. To guarantee a data-efficient inference of this quantity, we construct a Robust Gaussian Process (R-GP), where a deterministic GP realization of the Bayes risk can be obtained using the Sample Average Approximation (SAA) \citep{kleywegt2002sample, balandat2019botorch}.

\begin{figure*}[h]
\centering 
	\subfigure[]
	{ 
	\begin{minipage}{5cm}
		\centering 
\includegraphics[width=1.\textwidth]{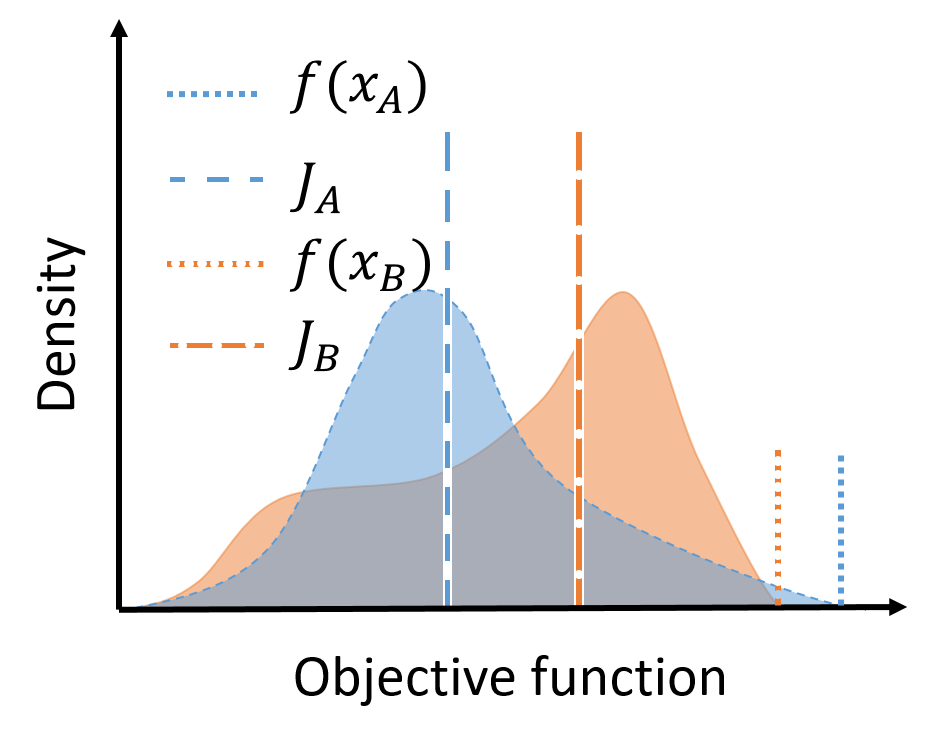} 
	\end{minipage}
	}
	\subfigure[]
	{ 
	\begin{minipage}{5cm}
		\centering 
		\includegraphics[width=1.1\textwidth]{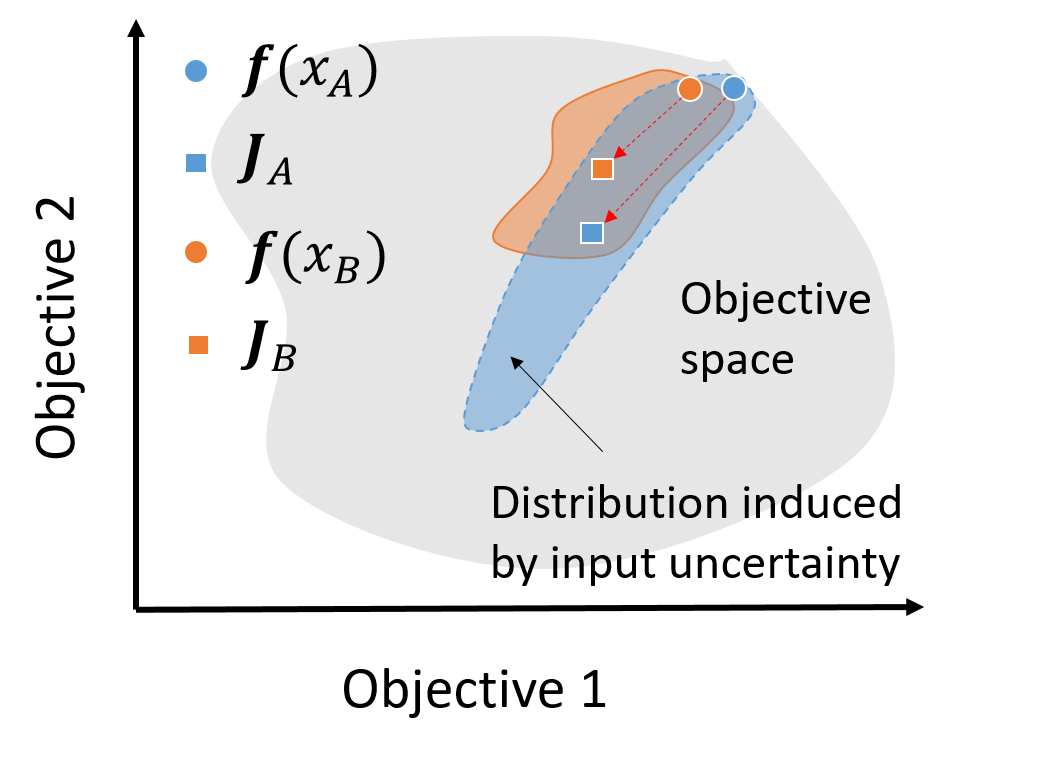} 
	\end{minipage}
	}
	\caption{Comparing two solutions in a maximization problem using their Bayes risk measures $J_A$ and $J_B$, i.e., the expectation of the objective function under uncertainty as defined in Eq. \ref{Eq: main_express} a) Single-objective: candidate $\boldsymbol{x}_B$ (orange) is superior over $\boldsymbol{x}_A$ (blue) as its Bayes risk measure $J_B$ is higher. b) Multi-objective: candidate $\boldsymbol{x}_B$ is preferable as $\boldsymbol{J}_B$ dominates the Bayes risk $\boldsymbol{J}_A$ of candidate $\boldsymbol{x}_A$.} 
\label{fig:idea_demo}	
\end{figure*}

Note that there is a mismatch in the type of uncertainty provided by the R-GP and the uncertainty expected by a common myopic acquisition function, as the latter usually implicitly assumes that this uncertainty comes from a random variable that is directly observable. In order to mitigate this issue, we propose a two-stage approach that can handle existing acquisition functions, including myopic acquisition functions which are commonly used in MOBO. The proposed flexible RMOBO-IU framework, illustrated in Fig. \ref{fig:RMOBO_flow} and detailed in Algorithm. \ref{Alg: rMOBO_alg}, can be used with existing acquisition functions and with different input uncertainty distributions. The effectiveness of this novel method has been demonstrated on several synthetic functions.

The key contributions can be highlighted as:
 \begin{enumerate}
    \item A \textbf{Bayesian optimization taxonomy} for robust multi-objective optimization. 
   \item A deterministic Robust Gaussian Process (R-GP), using the efficient Sample Average Approximation (SAA) based Monte Carlo kernel expectation approximation (SAA-MC KE) to infer the Bayes risk, with a proper complexity analysis. 
    \item We highlight some problems when applying a myopic acquisition function with a robust Gaussian Process and present a novel nested active learning policy to alleviate these.  
    \item New synthetic benchmark problems for robust multi-objective Bayesian optimization under input uncertainty. 
 \end{enumerate}

\begin{figure*}[h]
\centering 
\includegraphics[width=1.1\textwidth]{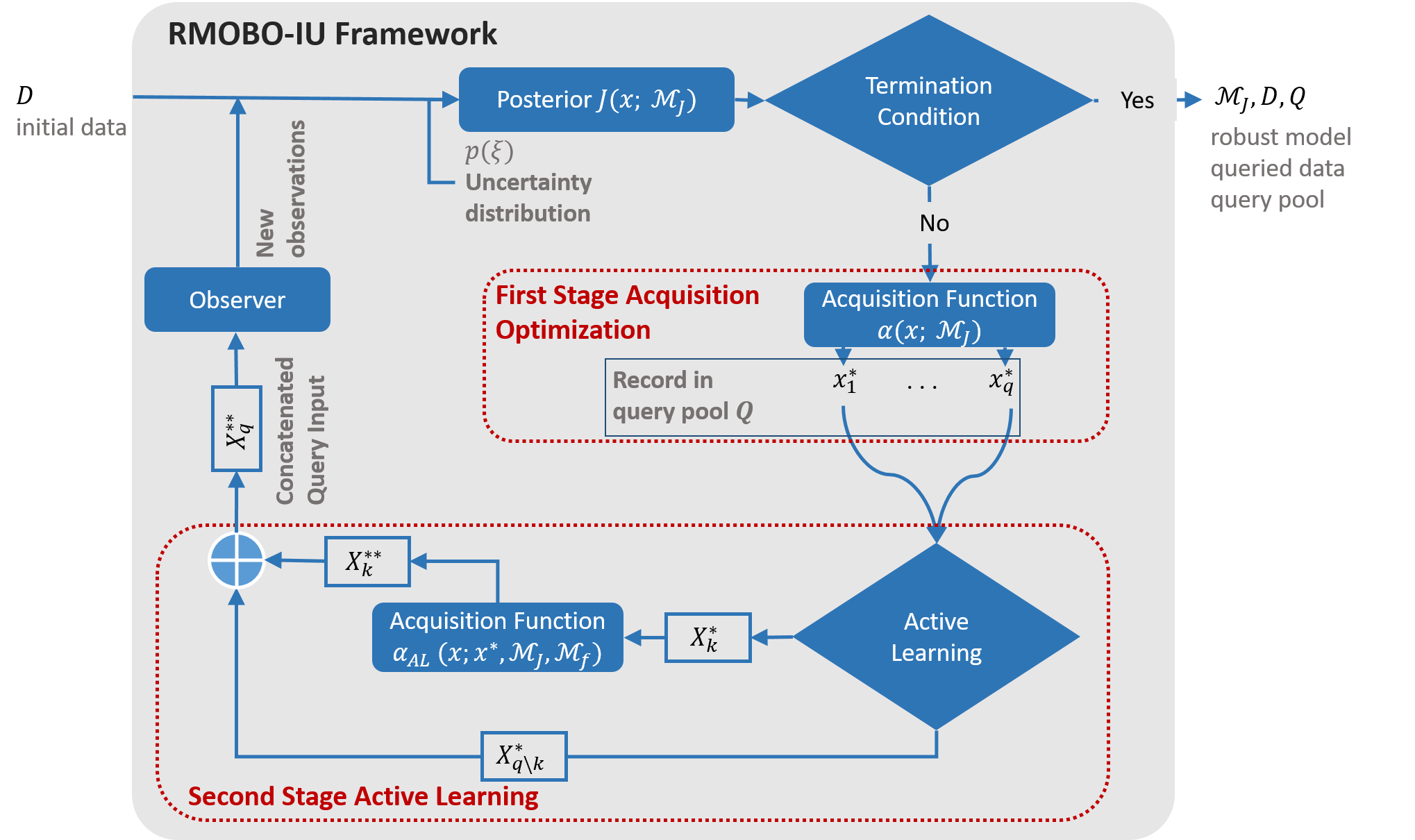} 
	\caption{RMOBO-IU flowchart: The Bayes risk $J$ of the objectives are derived from Gaussian processes $\mathcal{M}_{\boldsymbol{J}}$. In the first stage, an off-the-shelf acquisition function is used to select ($q$) query points: $\boldsymbol{X}_q^{*} := \{\boldsymbol{x}_1^{*}, ..., \boldsymbol{x}_q^{*}\}$. For the second stage, these points are then re-investigated by an Active Learning (AL) acquisition function if the AL process activation condition has been met. The updated query points are then fed to the expensive objective observer and augment the training data. The iteration loop continues until a termination condition has been met. Eventually, the robust model, queried data, and the query pool can be utilized to make optimal robust recommendations.} 
\label{fig:RMOBO_flow}	
\end{figure*}

The remaining of the paper is structured as follows. First, the background and related techniques are described in section 2. The RMOBO-IU framework, including the model description, is introduced in section 3. The numerical experiments are presented in section 4. Conclusions are provided in section 5.

\section{Preliminaries and Related Work}

\subsection{Preliminaries}

\noindent\textbf{Multi-Objective Optimization (MOO)} methods search for optimal solutions considering multiple objectives simultaneously. This can be mathematically expressed as finding the optimum of a vector-valued function $\boldsymbol{f}:=\{f_1, ..., f_M\}$ in a bounded design space $\mathcal{X} \subset \mathbb{R}^d $, where $M$ represents the number of objectives. In the context of MOO, the comparison of different candidates is done through a \textbf{ranking} mechanism $A_{rank}$. Considering the goal of $\textbf{maximizing}$ each objective function, a candidate $\boldsymbol{x}$ is preferable to $\boldsymbol{x}'$ if $\forall j \in M: f_j(\boldsymbol{x}) \geq f_j(\boldsymbol{x}')$ and $\exists j \in M: f_j(\boldsymbol{x}) > f_j(\boldsymbol{x}')$. This specific ranking strategy is termed as \textbf{dominance} ($\succ$) and described as $\boldsymbol{f}(\boldsymbol{x})$ dominates $\boldsymbol{f}(\boldsymbol{x}')$: $\boldsymbol{f}(\boldsymbol{x}) \succ \boldsymbol{f}(\boldsymbol{x}')$. In MOO, the candidate $\boldsymbol{x}$ is defined as \textbf{Pareto optimal input} if $\not\exists \boldsymbol{x}' \in \mathcal{X}$ such that $\boldsymbol{f}(\boldsymbol{x}') \succ \boldsymbol{f}(\boldsymbol{x})$. In this case, $\boldsymbol{f}(\boldsymbol{x})$ is defined as a \textbf{Pareto optimal point}. Given that different objectives usually conflict with each other, MOO seeks for a \textbf{Pareto frontier} $\mathcal{F}^*$ that consists of all the objective values $\boldsymbol{f}$ of the Pareto optimal solutions in the bounded design space $\mathcal{X}$: $\mathcal{F}^*:=\{\boldsymbol{f}\in \mathbb{F}_{\boldsymbol{f}} \vert \not\exists \boldsymbol{f}_{{\boldsymbol{x}'}} \in \mathbb{F}_{\boldsymbol{f}}\ s.t.\ \boldsymbol{f}_{\boldsymbol{x}'} \succ \boldsymbol{f} \}$, where $\mathbb{F}_{\boldsymbol{f}}: = \{\boldsymbol{x} \in \mathcal{X} \vert \boldsymbol{f}(\boldsymbol{x})\}$. 

In many scenarios, the vector-valued function $\boldsymbol{f}$ does not have a closed-form expression, and observing the function value may have a high computational cost. For this class of problems, it is of paramount importance to restrict the number of function queries when searching for $\mathcal{F}^*$.   

\noindent\textbf{Bayesian Optimization} (BO) \citep{jones1998efficient}  is a sequential
model-based approach to solving optimization problems efficiently \citep{shahriari2015taking}. Starting with a few training samples $D=\{\boldsymbol{X}, \boldsymbol{Y}\}$, it builds a Bayesian posterior model $\mathcal{M}$ (with a Gaussian Process (GP) as a common choice \citep{rasmussen2003gaussian}), as a computationally efficient \textbf{surrogate model} of $f$. Given the predictive distribution from the surrogate model, an \textbf{acquisition function} can be defined as a measure of informativeness for any point $\boldsymbol{x}$ in the design space. It is hence able to search and query the most informative candidate $\{\boldsymbol{x}, f(\boldsymbol{x})\}$ to augment the dataset $D$ and update $\mathcal{M}$ accordingly. This process of refining the posterior model and searching for optimal candidates can be conducted sequentially until a predefined stopping criterion has been met. Eventually, the final model $\mathcal{M}$ and the dataset $D$ can be utilized for recommending optimal solutions. The same paradigm is usually referred to as \textbf{Multi-Objective Bayesian Optimization (MOBO)} when $f$ is vector-valued, and the goal is searching for the Pareto frontier $\mathcal{F}^*$.

\noindent\textbf{Input Uncertainty} is a common type of uncertainty that is studied in this paper. Suppose we would like to implement a configuration $\boldsymbol{x}$. The input noise, which can be formulated as an additive noise term sampled from a distribution $\boldsymbol{\xi} \sim p(\boldsymbol{\xi})$, could result in a different implementation  $\boldsymbol{x + \xi}$ that can worsen the performance. The additive noise distribution $\boldsymbol{\xi}$ results in a distribution of possible objective function values $p(\boldsymbol{f}(\boldsymbol{x}+\boldsymbol{\xi})\vert \boldsymbol{\xi})$, which is refereed to as the \textbf{objective distribution}.

\subsection{Related Work}

Several approaches have been proposed to link the robust MOO methodology with a GP surrogate model \citep{xia2014utilizing, zhou2018multi, rivier2018surrogate, abbas2016multiobjective}. \cite{xia2014utilizing} consider the worst-case robustness scenario, for which the worst objective function is extracted from the GP. \cite{zhou2018multi} introduce a GP surrogate model assisted multi-objective robust optimization strategy based on \cite{10.1115/1.2202884}, where the GP acts as an efficient surrogate and hence, as a cheap intermediary for a genetic algorithm to search for the optimum. In a more probabilistic setting, \cite{rivier2018surrogate} propose an interesting bounding box-based efficient MOO framework. For each observation, a conservative bounding box is constructed based on some robustness measures approximated by MC sampling on the surrogate model, with the assumption that an extra aleatory variable can be modeled with a uniform distribution built upon the bounding box. The concept of probability of box-based Pareto dominance is utilized to compare against different aleatory variables hence different observations. Subsequently, it can search for the optimum or improve the surrogate model accuracy accordingly. Nevertheless, while equipped with a GP as a probabilistic surrogate model, the robustness measure of the above-mentioned approaches are usually extracted in a non-Bayesian way as a point estimation from the posterior mean, and the surrogate model refinement step must be defined explicitly. A more principled BO-like RMOBO framework has yet to be revealed.

\section{RMOBO-IU Framework}  
\subsection{Optimizing Bayes Risk versus Optimizing the Original Objective Function} \label{seq: difference_pareto}
The Bayes risk is utilized as objective in the RMOBO-IU framework: 

\begin{equation}
\begin{aligned}
& \underset{\boldsymbol{x} \in \mathcal{X} \subset \mathbb{R}^d}{\text{Maximize}}\ J_1(\boldsymbol{x}), J_2(\boldsymbol{x}), ..., J_M(\boldsymbol{x})\\
& \text{where}\ J(\boldsymbol{x}) = \int f(\boldsymbol{x}+\boldsymbol{\xi}) p(\boldsymbol{\xi}) d\boldsymbol{\xi}\\
\label{Eq: main_express}
\end{aligned}
\end{equation}


Given the fact that we are optimizing the Bayes risk $\boldsymbol{J}$, we use $\mathcal{F}_{\boldsymbol{J}}^*$ and $\mathcal{F}_{\boldsymbol{f}}^*$ to represent the Pareto frontier of the robust and non-robust optimization problem (i.e., optimize the original objective function $\boldsymbol{f}$), respectively. It is natural to wonder what the difference is between $\mathcal{F}_{\boldsymbol{f}}^*$ and $\mathcal{F}_{\boldsymbol{J}}^*$. Using the objective space, the difference can be categorized into four different cases \citep{deb2005searching} as shown in Fig. \ref{fig: different_pf}. Except for the first case, the remaining cases clearly show that $\mathcal{F}_{\boldsymbol{J}}^*$ leads to more robust optimal solutions, at least for some parts of the Pareto fronts.

\begin{figure*}[h]
\centering 
	\subfigure[$\mathcal{F}_{\boldsymbol{f}}^*$ maps directly to $\mathcal{F}_{\boldsymbol{J}}^*$.]{ 
	\begin{minipage}{5cm}
		\centering 
		\includegraphics[width=1\textwidth]{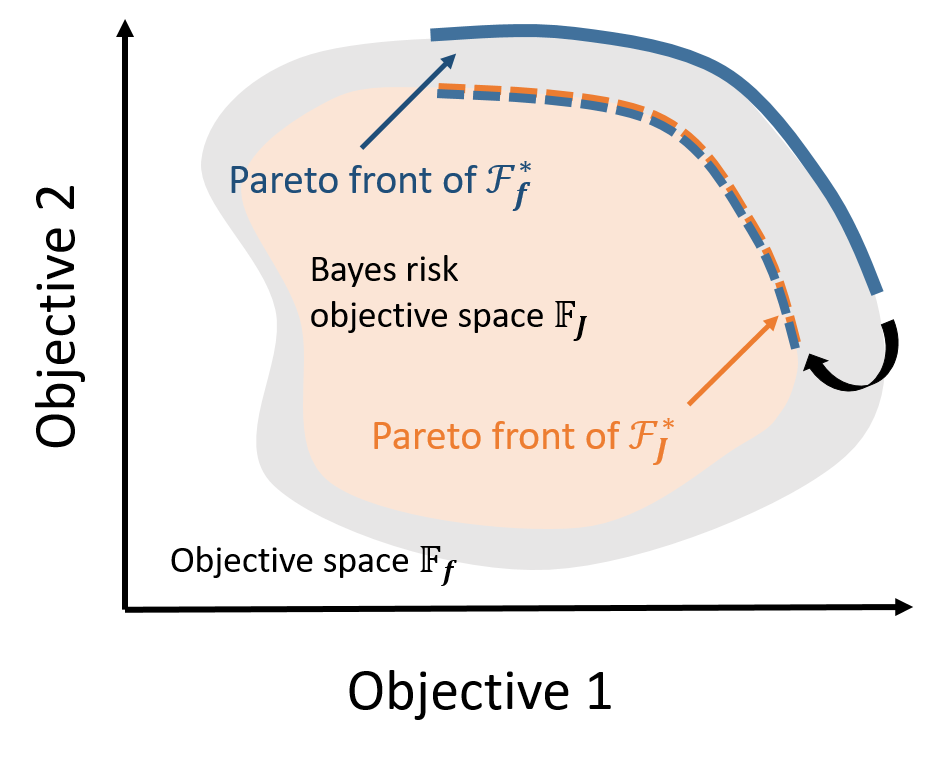} 
	\end{minipage}
	}
	\subfigure[Part of $\mathcal{F}_{\boldsymbol{f}}^*$ remains as $\mathcal{F}_{\boldsymbol{J}}^*$.]{ 
	\begin{minipage}{5cm}
		\centering 
		\includegraphics[width=1\textwidth]{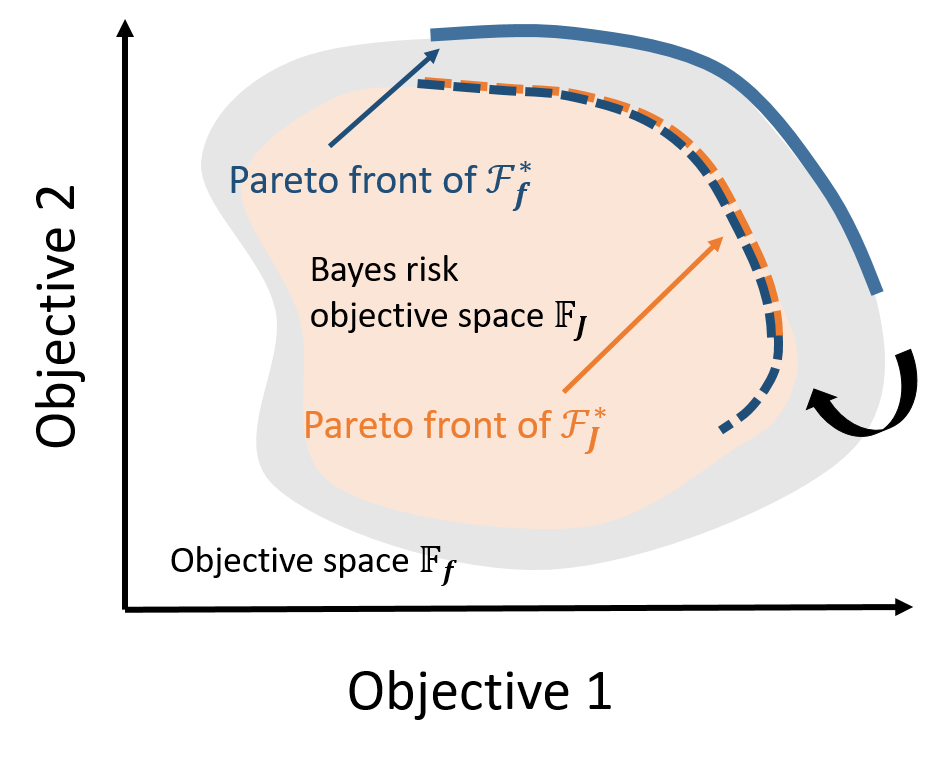} 
	\end{minipage}
	}
	\subfigure[Part of $\mathcal{F}_{\boldsymbol{f}}^*$ remains in $\mathcal{F}_{\boldsymbol{J}}^*$.]{ 
	\begin{minipage}{5cm}
		\centering 
		\includegraphics[width=1\textwidth]{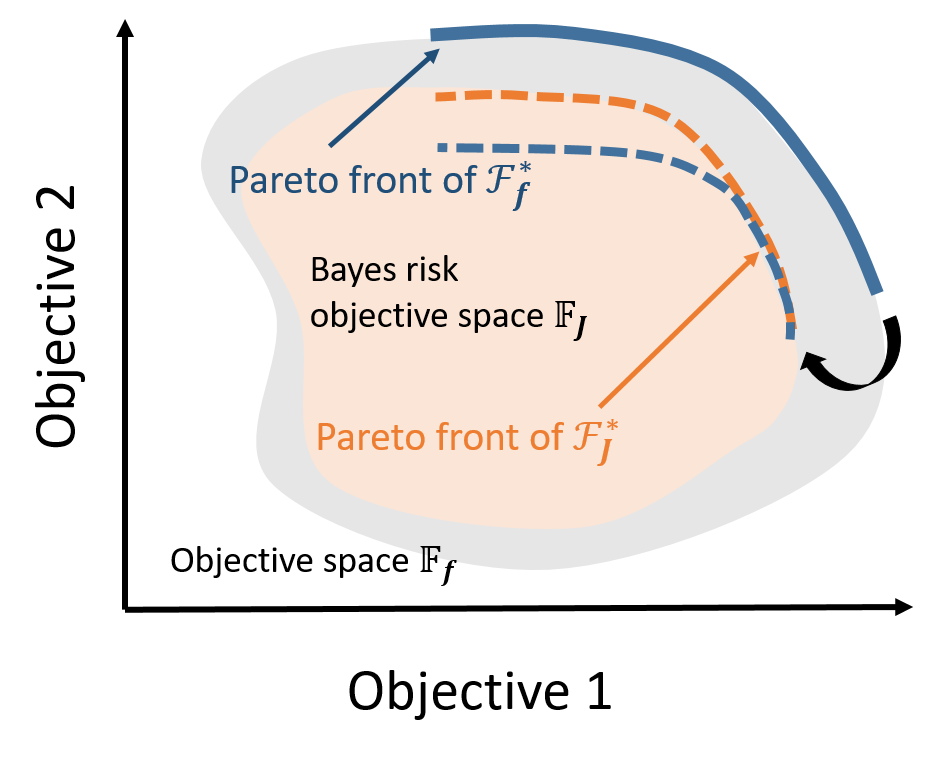} 
	\end{minipage}
	}
	\subfigure[$\mathcal{F}_{\boldsymbol{f}}^*$ is not robust.]{ 
	\begin{minipage}{5cm}
		\centering 
		\includegraphics[width=1\textwidth]{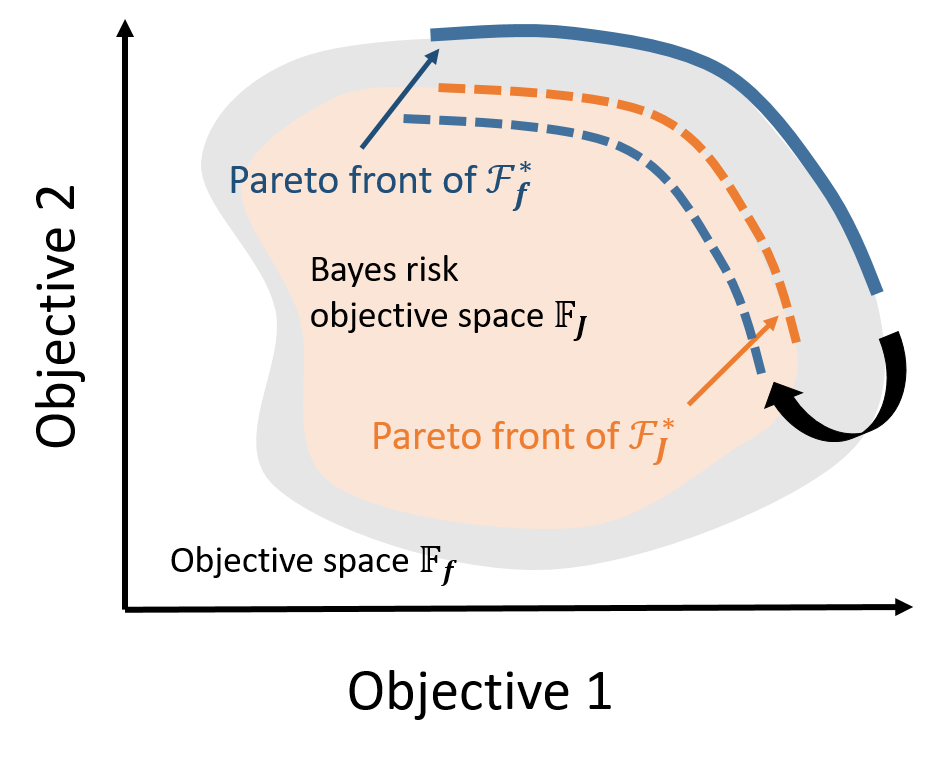} 
	\end{minipage}
	}
	\caption{Four different cases comparing the Pareto front $\mathcal{F}_{\boldsymbol{f}}^*$ with the Pareto front $\mathcal{F}_{\boldsymbol{J}}^*$ considering input uncertainty. Images courtesy of \cite{deb2005searching}.} 
\label{fig: different_pf}	
\end{figure*}

It might be difficult to determine whether a robust Pareto front exists that is different from $\mathcal{F}_f^*$. This is not trivial to answer due to the agnostic property of the black-box function $\boldsymbol{J}$. Nevertheless, from a practitioner perspective, we define a sufficient condition based on the objective functions which helps to determine whether a distinct robust Pareto front exists:

\begin{prop}
If $\ \exists f_i \in \boldsymbol{f}, J_i \in \boldsymbol{J}$, s.t. for $ \boldsymbol{x}_{f_i^*} : = \underset{\boldsymbol{x} \in \mathcal{X}}{argmax}\       f_i(\boldsymbol{x})$,  $\boldsymbol{x}_{J_i^*} : = \underset{\boldsymbol{x} \in \mathcal{X}}{argmax}\  J_i(\boldsymbol{x})$.  $\boldsymbol{x}_{f_i^*} \neq \boldsymbol{x}_{J_i^*}$, $\boldsymbol{x}_{J_i^*}$ is unique and $\boldsymbol{f}_{\boldsymbol{x}_{J_i^*}} \not\in \mathcal{F}_{\boldsymbol{f}}^*$.  

\noindent Then: \\ $\exists \boldsymbol{x}_{diff} \in \mathcal{X}$ such that $\boldsymbol{J}(\boldsymbol{x}_{diff}) \in \mathcal{F}_{\boldsymbol{J}}^*$ while $\boldsymbol{f}(\boldsymbol{x}_{diff}) \not\in \mathcal{F}_{\boldsymbol{f}}^*$, and  $\mathcal{F}_{\boldsymbol{f}}^* \neq \mathcal{F}_{\boldsymbol{J}}^*$.
\label{prop: usefulness}
\end{prop}
\begin{proof}
Given $\boldsymbol{x}_{f_i^*} : = \underset{x \in \mathcal{X}}{argmax}\  f_i(\boldsymbol{x})$,  $\boldsymbol{x}_{J_i^*} : = \underset{\boldsymbol{x} \in \mathcal{X}}{argmax}\  J_i(\boldsymbol{x})$, having $\boldsymbol{x}_{f_i^*} \neq \boldsymbol{x}_{J_i^*}$ means $f_i(\boldsymbol{x}_{f_i^*}) > f_i(\boldsymbol{x}_{J_i^*})$ and $J_i(\boldsymbol{x}_{f_i^*}) < J_i(\boldsymbol{x}_{J_i^*})$, according to the definition of Pareto dominance, Let $\boldsymbol{x}_{diff}:=\boldsymbol{x}_{J_i^*}$, we have $\boldsymbol{J}(\boldsymbol{x}_{diff}) \in \mathcal{F}_{\boldsymbol{J}}^*$ and $\boldsymbol{f}(\boldsymbol{x}_{diff}) \not\in \mathcal{F}_{\boldsymbol{f}}^*$. Meanwhile, as
$\nexists \boldsymbol{x} \in \{\boldsymbol{x} \in \mathcal{X}\vert \boldsymbol{f}(\boldsymbol{x}) \in \mathcal{F}_{\boldsymbol{f}}^*\}$ such that the $i$th component of its outcome: $J_i(\boldsymbol{x}) \geq J_i(\boldsymbol{x}_{diff})$, hence $\mathcal{F}_{\boldsymbol{f}}^* \neq \mathcal{F}_{\boldsymbol{J}}^*$ and the proposition holds. 
\end{proof}

The proposition conveys that if the objective function $f_i$ has a different global maximum location $\boldsymbol{x}_{J_i^*}$ (for Bayes risk) which is also not Pareto optimal in the objective space, then there will be a distinct robust Pareto front.


\subsection{Inference of the Bayes Risk} \label{Sec:Infer_J}
\subsubsection{Robust Gaussian Process (R-GP)}

\begin{figure*}[!htb]
\centering 
	\subfigure{ 
	\begin{minipage}{10cm}
		\centering 
		\includegraphics[width=1.\textwidth]{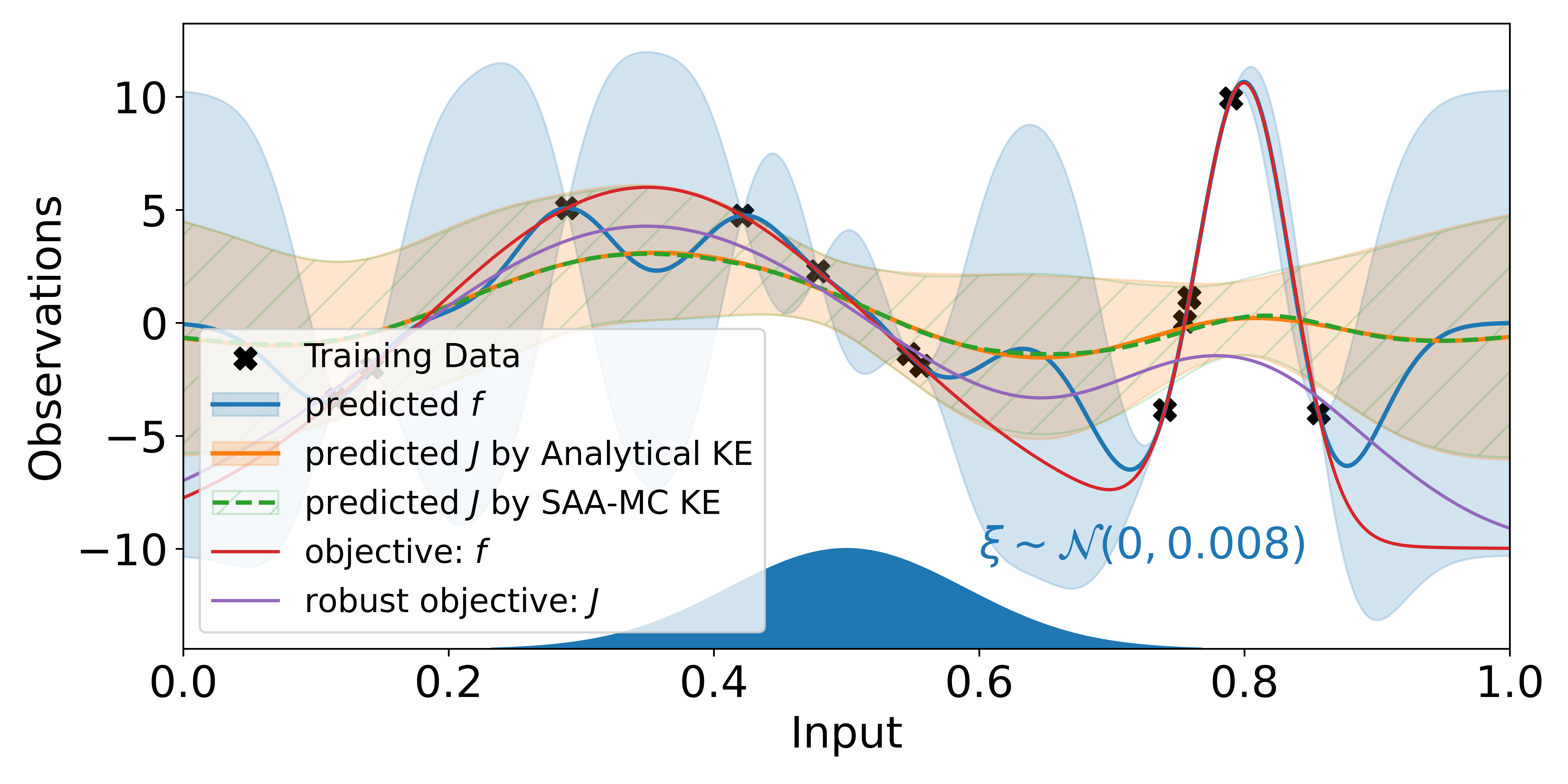} 
	\end{minipage}
	\label{Fig: J_compare}
	}
	\caption{1D example of a standard GP (blue) and Robust GP (orange and green) considering a Gaussian input uncertainty distribution (with variance 0.008), illustrated at input $x=0.5$. It can be observed that the Bayes risk favors the less risky maxima. The comparison of the posterior distribution using analytical KE (orange) and SAA-MC based KE (green) GP is also shown. The SAA-MC KE results in a differentiable approximation of the GP posterior. } 
\label{fig:GP_J_pred}	
\end{figure*}

Given limited training data $D$, we specify an independent GP prior on each black-box function $f$. Hence, the $i$th GP model $\mathcal{M}_{f_i}$'s posterior representing $f_i$ at $\boldsymbol{x}$ is: 

\begin{equation}
m_f(\boldsymbol{x}\vert D) = \boldsymbol{k}_{f}(\boldsymbol{x})^T\boldsymbol{K}^{-1}\boldsymbol{y}
\end{equation}

\begin{equation}
Cov_{f} (\boldsymbol{x}\vert D) = \boldsymbol{k}_{f}(\boldsymbol{x}, \boldsymbol{x}') - \boldsymbol{k}_{f}(\boldsymbol{x})^T\boldsymbol{K}^{-1}\boldsymbol{k}_{f}(\boldsymbol{x}')
\end{equation}

\noindent where $\boldsymbol{K}$ is the kernel matrix of observations.

Now consider the transformation of expectation through input uncertainty. Since the expectation in Eq. \ref{Eq: main_express} is a linear operator, we can derive a \textbf{robust GP} for the Bayes risk $J_i$ by applying linear transformation rules \citep{rasmussen2003gaussian, papoulis2002probability}, resulting in: 
\begin{equation}
p(J|D, \boldsymbol{x}) = \mathcal{N}(m_J, Cov_J)
\end{equation}
\begin{equation}
m_J(\boldsymbol{x}\vert D) = \boldsymbol{k}_{Jf}(\boldsymbol{x})^T\boldsymbol{K}^{-1}\boldsymbol{y}
\label{Eq: J mean}
\end{equation}
\begin{equation}
Cov_J(\boldsymbol{x}\vert D) = \boldsymbol{k}_{J}(\boldsymbol{x}, \boldsymbol{x}') - \boldsymbol{k}_{Jf}(\boldsymbol{x})^T\boldsymbol{K}^{-1}\boldsymbol{k}_{fJ}(\boldsymbol{x}')
\label{Eq. prd_J_cov}
\end{equation}

\noindent where $\boldsymbol{k}_{Jf}$ and $\boldsymbol{k}_{J}$ are defined using the following Kernel Expectation (KE):

\begin{equation}
\boldsymbol{k}_{Jf} (\boldsymbol{x})= \int k_f(\boldsymbol{x}+ \boldsymbol{\xi})p(\boldsymbol{\xi})d\boldsymbol{\xi}
\label{Kjf}
\end{equation}

\begin{equation}
\begin{aligned}
\boldsymbol{k}_J (\boldsymbol{x}, \boldsymbol{x}')&= \int \int k_f(\boldsymbol{x}+ \boldsymbol{\xi}, \boldsymbol{x}'+ \boldsymbol{\xi}')p(\boldsymbol{\xi})p(\boldsymbol{\xi}')d\boldsymbol{\xi}d\boldsymbol{\xi}'
\end{aligned}
\label{K_J}
\end{equation}

For some kernels and uncertainty distributions $p(\boldsymbol{\xi})$, an analytical expression exists for the KE. One of the most well-known analytical KE is the squared exponential kernel under Gaussian input uncertainty \citep{dallaire2009learning}, see Fig. \ref{Fig: J_compare} (orange posterior mean and uncertainty interval). Unfortunately, for more generic cases, an analytical expression is non-trivial to obtain. In this case, one can defer to Monte Carlo (MC) approximations\footnote{To improve the numerical stability, we leverage the methodology of \cite{higham1988computing} with a nugget term to search for the nearest positive definite matrices for Eq. \ref{Eq: rGP_Cov} when a full covariance posterior matrix is needed.}:
\begin{equation}
\begin{aligned}
m_J(\boldsymbol{x}\vert D) & \approx \frac{1}{N}\sum_{i=1}^N \left[ k_f(\boldsymbol{x}+ \boldsymbol{\xi}_i)\right]\boldsymbol{K}^{-1}\boldsymbol{y}
\label{Eq: rGP_mean}
\end{aligned}
\end{equation}

\begin{equation}
\begin{aligned}
Cov_J(\boldsymbol{x}\vert D) & \approx \frac{1}{N} \sum_{i=1}^N \left[k_f(\boldsymbol{x} + 
\boldsymbol{\xi}_i, \boldsymbol{x}' + \boldsymbol{\xi}'_i)  -  k_f(\boldsymbol{x}+ \boldsymbol{\xi}_i, \boldsymbol{x}')\boldsymbol{K}^{-1} k_f(\boldsymbol{x}, \boldsymbol{x}'+ \boldsymbol{\xi}'_i) \right] 
\label{Eq: rGP_Cov}
\end{aligned}
\end{equation}

While the common approach is to redraw samples $\boldsymbol{\xi}$ for every evaluation point $\boldsymbol{x}$ to obtain the posterior predictive distribution, we apply the sample average approximation \citep{kleywegt2002sample, balandat2019botorch} through the MC based kernel expectation (SAA-MC KE). This is illustrated in Fig. \ref{fig:GP_J_pred} (green posterior mean and uncertainty interval). Given a differentiable kernel, by holding MC samples fixed: $E = \{\boldsymbol{\xi}^1, ..., \boldsymbol{\xi}^N\}$ for KE,  we are able to provide a deterministic and differentiable approximation of the posterior distribution, which is easily utilizable by off-the-shelf acquisition functions. Furthermore, we can still use gradient-based optimizers for optimizing the acquisition function.

\subsubsection{Inference Complexity}

\begin{table*}[h]
\caption{Inference complexity of a standard GP and R-GP, where $n_{tr}$ is the training sample size and $n_{test}$ is the test sample size. $N$ is the MC sample size for the kernel expectation.}
\begin{center}
\begin{tabular}{llll}
\hline 
& standard GP  & R-GP \\
\hline
 Computation  & & \\
 Complexity & & \\
\hline
Not Full-Cov Inference   & $n_{tr}n_{test}$  & $  N \cdot n_{tr}n_{test}$\\
Full-Cov Inference & $n_{tr}n_{test}^2$  & $N \cdot n_{tr}n_{test}^2$ \\ 
\hline
Memory   \\
Consumption (Parallized)\\
\hline
Not Full-Cov Inference   & $max(n_{test}n_{tr})$    & 
$N \cdot max( n_{test}n_{tr})$\\
Full-Cov Inference & $max(n_{test}^2, n_{test}n_{tr})$ & $N \cdot max( n_{test}^2,  n_{test}n_{tr})$\\
\hline\\
\end{tabular}
\end{center}
\label{Tab:acquisition_evaluation_time}
\end{table*}

We derive the computation complexity of inferencing the Bayes risk $J$ with respect to the test sample size $n_{test}$, as well as the memory consumption\footnote{We report the single storage component that can possibly take the maximum memory, and we do not consider the memory consumption for the original kernel matrix storage as it is not correlated with $n_{test}$.} in Table. \ref{Tab:acquisition_evaluation_time}. Fortunately, the main extra computation effort only affects the inference stage instead of the model training stage. The latter is usually regarded as the main bottleneck of GPs. For common GP implementations, the introduction of MC samples increases the  complexity $N$ times, i.e., it grows linear with the number of MC samples. We propose to parallelize the computation through $N$ MC samples and so, we trade of the time increment against memory consumption. 


\subsection{Two-Stage Acquisition Function Optimization Process} \label{Sec:Acq}
\subsubsection{First Stage: Acquisition Optimization}
As the R-GP provides a (multivariate) normal posterior distribution $p(\boldsymbol{J} \vert \boldsymbol{x}, D)$, it is convenient to utilize existing (multi-objective) acquisition functions to search for the Pareto frontier $\mathcal{F}_{\boldsymbol{J}}^*$. We use common myopic acquisition functions for MOBO (e.g., Expected Hypervolume Improvement (EHVI) \citep{yang2019efficient}, Parallel Expected Hypervolume Improvemet (qEHVI) \citep{daulton2020differentiable, daulton2021parallel} and Expected Hypervolume Probability of Improvement (EHPI) \citep{yang2019efficient, couckuyt2014fast}), with a brief remark below. 

Recall that many acquisition functions can be written in the following form \citep{wilson2018maximizing}:

\begin{equation}
    \alpha(\boldsymbol{X}_q; \psi, D)  = \int_{\boldsymbol{J}_{\boldsymbol{X}_q}}\mathcal{\ell}(\boldsymbol{J}_{\boldsymbol{X}_q}; \psi) p(\boldsymbol{J}_{\boldsymbol{X}_q}; m_{J}, Cov_{J})d\boldsymbol{J}_{\boldsymbol{X}_q}
\label{Eq: myopic_acq}
\end{equation}
\noindent where $\ell$ denotes the utility function (using the acquisition function parameter $\psi$), $\boldsymbol{X}_q:=\{\boldsymbol{x}_1, ..., \boldsymbol{x}_q\}$ represents a batch of $q$ input candidates. For myopic acquisition functions, $\psi$ can be defined as the \textit{current best Pareto frontier} inferred using the R-GPs: $\psi := A_{rank}(\boldsymbol{J}_D \vert \mathcal{M}_{\boldsymbol{J}}, D)$, and results in the following expression:

\begin{equation}
\begin{aligned}
    \alpha(\boldsymbol{X}_q; \psi, D)  & =  \int_{\boldsymbol{J}_D}\int_{\boldsymbol{J}_{\boldsymbol{X}_q}} \mathcal{\ell}(\boldsymbol{J}_{\boldsymbol{X}_q}; A_{rank}(\boldsymbol{J}_D \vert \mathcal{M}_{\boldsymbol{J}}, D)) \\ &p(\boldsymbol{J}_{\boldsymbol{X}_q}; m_{\mathcal{\boldsymbol{J}}_{\boldsymbol{X}_q}}, Cov_{\mathcal{\boldsymbol{J}}_{\boldsymbol{X}_q}})p(\boldsymbol{J}_D;m_{\mathcal{\boldsymbol{J}}_{D}}, Cov_{\mathcal{\boldsymbol{J}}_{D}})d\boldsymbol{J}_{\boldsymbol{X}_q}d\boldsymbol{J}_D\\ & \approx   \int_{\boldsymbol{J}_{\boldsymbol{X}_q}}\mathcal{\ell}(\boldsymbol{J}_{\boldsymbol{X}_q}; A_{rank}(\overline{\boldsymbol{J}_D} \vert \mathcal{M}_{\boldsymbol{J}}, D)) p(\boldsymbol{J}_{\boldsymbol{X}_q}; m_{\mathcal{\boldsymbol{J}}_{\boldsymbol{X}_q}}, Cov_{\mathcal{\boldsymbol{J}}_{\boldsymbol{X}_q}})d\boldsymbol{J}_{\boldsymbol{X}_q}
\end{aligned}
\label{Eq: myopic_acq_2}
\end{equation}

\noindent The operator $ A_{rank}$ is the non-dominated sorting operation. We note that the extracted current best Pareto frontier is also a distribution due to the fact that the Bayes risk $\boldsymbol{J}$ is not observable. We could simplify the problem by making use of the posterior mean of the R-GP: $A_{rank}(\overline{\boldsymbol{J}_D} \vert \mathcal{M}_{\boldsymbol{J}}, D)$ as an approximation to avoid the integration of Pareto frontier distribution \citep{gramacy2010optimization}, resulting in the last line of Eq. \ref{Eq: myopic_acq_2}. Nevertheless, the distribution of the Pareto frontier can also be considered, for instance, by leveraging MC sampling \citep{daulton2021parallel}. Finally, we remark the last line of Eq. \ref{Eq: myopic_acq_2} can be analytically calculated exactly for EHVI, EHPI, and approximately calculated by qEHVI acquisition functions.

\begin{figure*}[h]
\centering 
\includegraphics[width=0.6\textwidth]{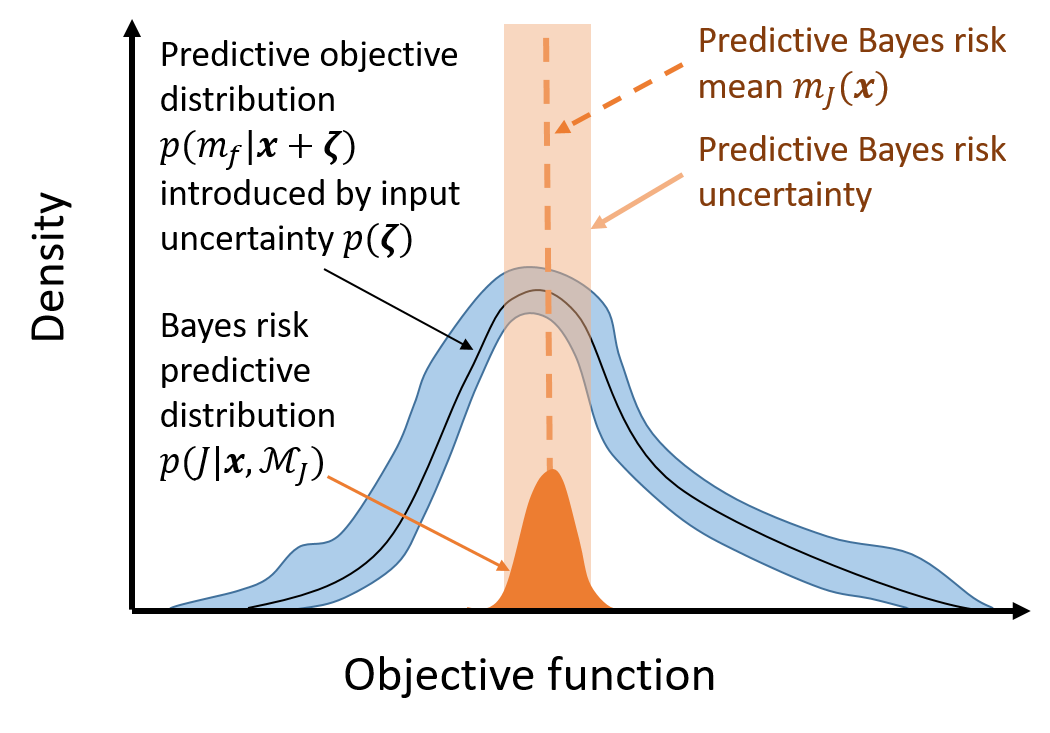}
\caption{Illustration of the non-observable property of the posterior distribution of R-GP at input location $\boldsymbol{x}$. The R-GP prediction of the Bayes risk's uncertainty of is illustrated as the orange shaded area. The blue shaded area represents the uncertainty of the predictive objective distribution which comes from the model $\mathcal{M}_f$. This implies since we do not have direct observation $\{\boldsymbol{x}, \boldsymbol{J}(\boldsymbol{x})\}$, only when the blue shaded area is reduced by sampling the Bayes risk predictive uncertainty will be reduced to zero.} 
\label{Fig: J_uncertainty}
\end{figure*}

\subsubsection{Second Stage: Active Learning for Reducing Uncertainty} \label{Sec: myopic_issues} 

While we can already use the acquisition function to search for the Pareto front $\mathcal{F}_{\boldsymbol{J}}^*$, we note that there is an inconsistency between the R-GP's inference $p(\boldsymbol{J} \vert \boldsymbol{x}, D)$ and what the acquisition functions mentioned above expects. More specifically, as illustrated in Fig. \ref{Fig: J_uncertainty}, the predict variance of the R-GP posterior aggregates uncertainty coming from the input uncertainty and the model approximations  $\mathcal{M}_{\boldsymbol{f}}$. This means the inferred Bayes risk $\boldsymbol{J} \vert \boldsymbol{x}, D$ could still be uncertain (i.e., the predict variance of $p(\boldsymbol{J} \vert \boldsymbol{x}, D)$ doesn't vanish to zero) at $\boldsymbol{x}$ even if the model has already included data at $\boldsymbol{x}$. Nevertheless, the myopic acquisition function, which build on the assumption that its predictive quantity $p(\boldsymbol{J} \vert \boldsymbol{x}, D)$ to be directly observable \citep{iwazaki2021mean, frohlich2020noisy}, cannot handle this inconsistency intrinsically. 

This results in two possible issues when applying of standard BO. First, as the input uncertainty could result in a design that is outside the bounded design space, the inference variance of the Bayes risk cannot be lowered to zero when restricting sampling only inside the design space. This results in the acquisition function adding duplicate samples at boundary locations. Secondly, common acquisition functions will waste resources on the same sample within the design space in a futile effort to reduce uncertainty, resulting in another duplication issue, which can also impose numerical instabilities to the model.

While not explicitly discussed in most of the existing research, we remark that these issues generically exist in single-objective robust BO when performing optimization on the Bayes risk. In order to resolve these issues, we propose an AL policy. We introduce an information-theoretic-based active learning acquisition function. As illustrated in Fig. \ref{Fig: AL_acq}, its intuitive interpretation is that we want to maximally reduce the uncertainty of the predictive distribution of $\boldsymbol{J} \vert D$ at candidate $\boldsymbol{x}^*$. Instead of directly sampling at $\boldsymbol{x}^*$, we seek the candidate that can maximally reduce its uncertainty, which is quantified by differential entropy.


\begin{equation}
    \alpha_{AL} = \mathbb{H}[\boldsymbol{J}(\boldsymbol{x^*} \vert D)] - \mathbb{E}_{\boldsymbol{f}(\boldsymbol{x})}\mathbb{H}[\boldsymbol{J}(\boldsymbol{x^*}\vert D, \{\boldsymbol{x}, \boldsymbol{f}(\boldsymbol{x})\})]
\end{equation}

\noindent where the expectation is taken through all possible $\boldsymbol{f}(\boldsymbol{x})$ described by the GP posterior. Given the assumption that we fixed the GP model $\mathcal{M}$'s hyperparameters during the acquisition optimization, the variance of $\boldsymbol{J}(\boldsymbol{x^*}\vert D, \{\boldsymbol{x}, \boldsymbol{f}(\boldsymbol{x})\})$ is independent of $\boldsymbol{f}(\boldsymbol{x})$, and hence it is sensible to avoid the expensive computation of the one dimensional integration by only making use of the posterior mean of $\boldsymbol{f}(\boldsymbol{x})$:  

\begin{equation}
\begin{aligned}
    \alpha_{AL} &\approx \mathbb{H}[\boldsymbol{J}(\boldsymbol{x}^* \vert D)] - \mathbb{H}[\boldsymbol{J}(\boldsymbol{x}^*\vert D, \{\boldsymbol{x}, \overline{\boldsymbol{f}}(\boldsymbol{x})\})] \\ & = \frac{1}{2} \text{log}\frac{\prod_{i=1}^M\mathbb{V}_{J_i}(\boldsymbol{x}^*\vert D)}{\prod_{i=1}^M\mathbb{V}_{J_i}(\boldsymbol{x}^*\vert D,\{\boldsymbol{x}, \overline{\boldsymbol{f}}(\boldsymbol{x})\})}
\end{aligned}
\end{equation}

\begin{figure*}[h]
\centering 
    \subfigure[Before the AL process]
	{ 
	\begin{minipage}{5.5cm}
		\centering 
\includegraphics[width=1\textwidth]{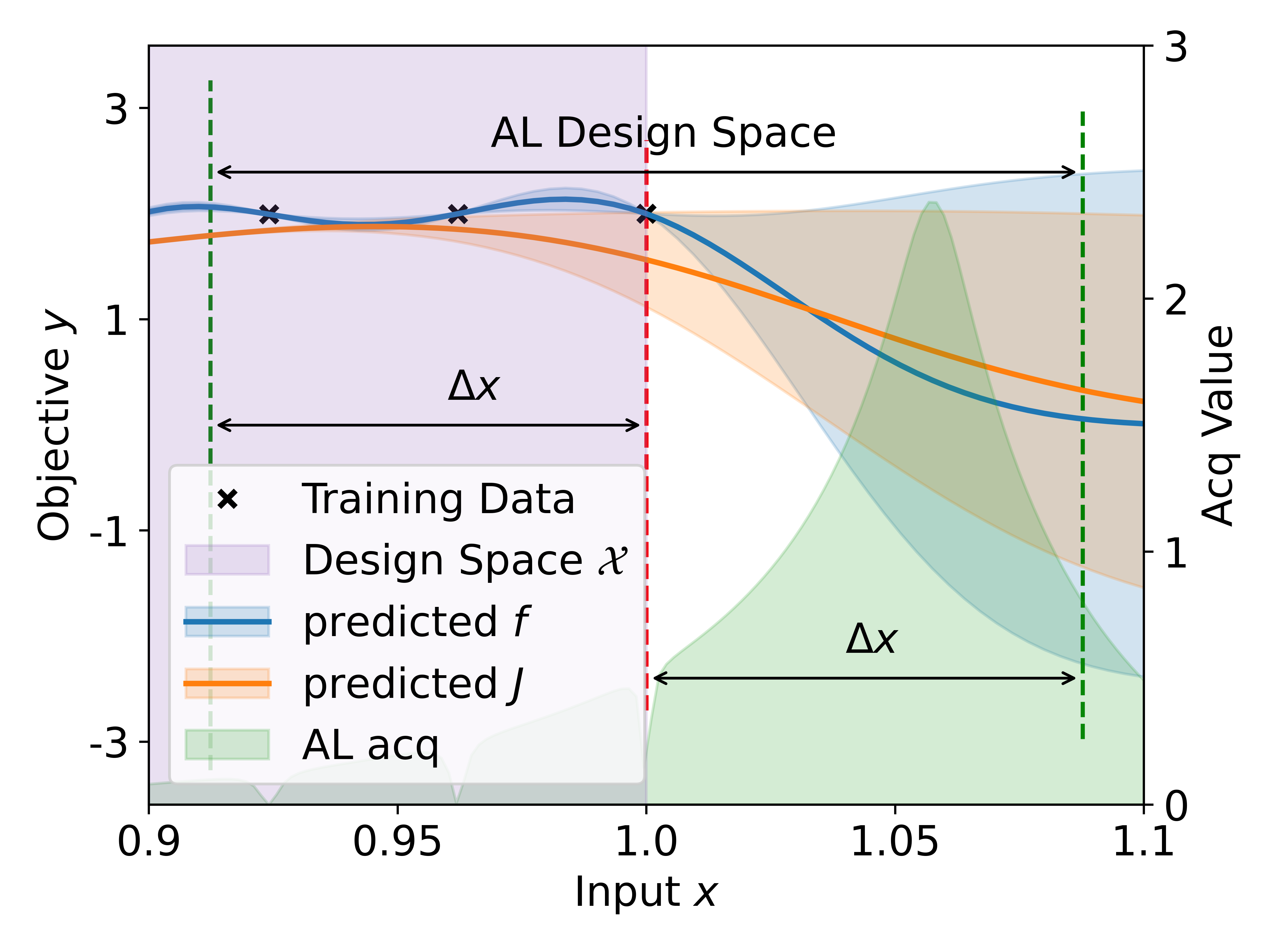}
	\end{minipage}
	}
	\subfigure[After the AL process]
	{ 
	\begin{minipage}{5.5cm}
		\centering 
		\includegraphics[width=1\textwidth]{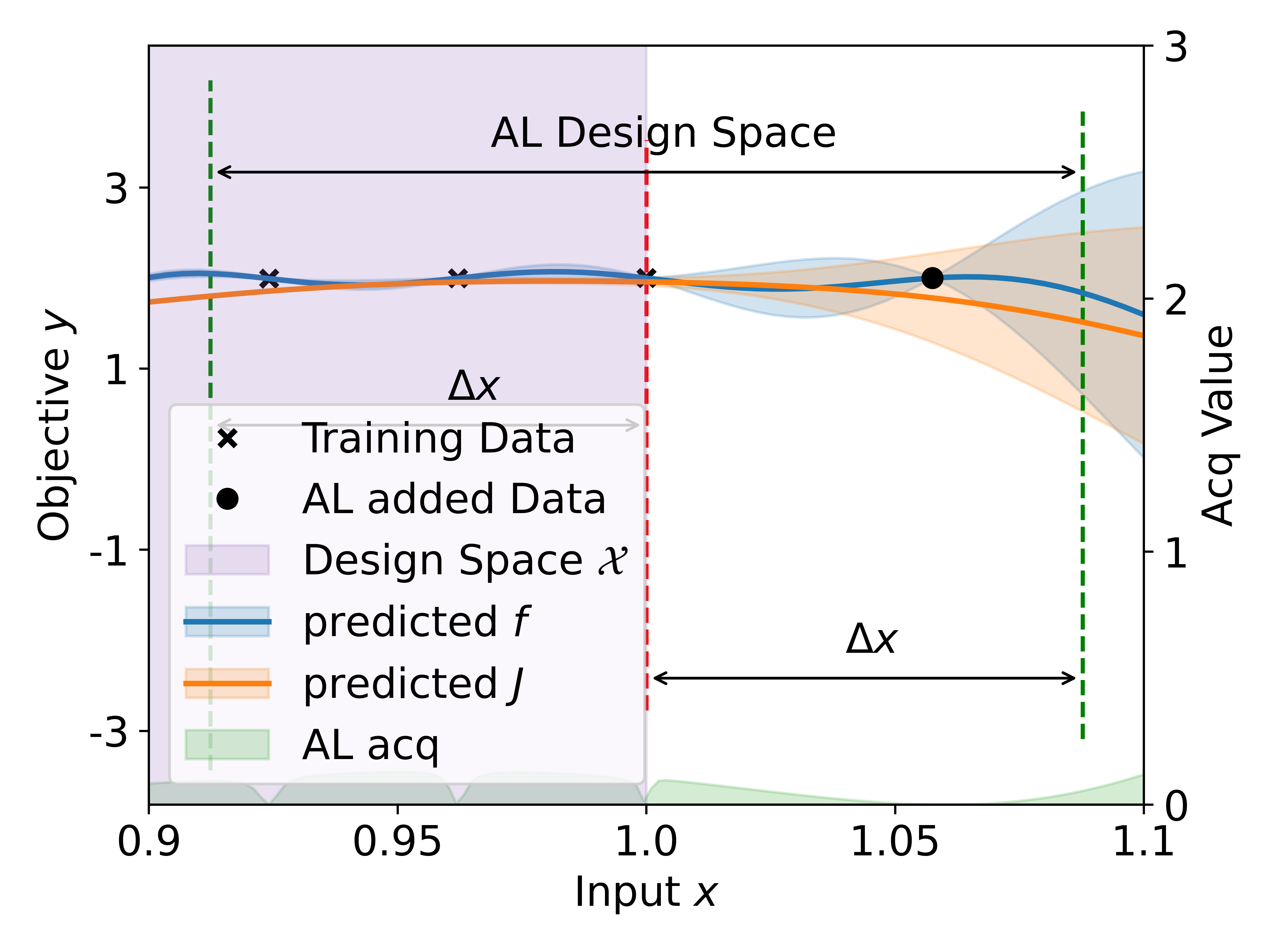} 
	\end{minipage}
	}
	\caption{Illustration of the \textit{boundary issue} and the \textit{duplication issue} in robust optimization. The predictive uncertainty of Bayes risk at $x=1$ (i.e., $\mathbb{V}(J(x=1))$) cannot be lowered to zero even by sampling exactly at this location. The active learning (green) acquisition function is proposed to resolve the \textit{boundary issue}. The design space boundary is illustrated as the red vertical dashed line. The AL process can hence result in sampling outside the original design space in order to reduce the uncertainty at the design space boundary.} 
\label{Fig: AL_acq}
\end{figure*}

\noindent where $\mathbb{V}_{J_i}$ represents the variance of $i$th $J$. As AL  brings extra computational complexity, it is sensible to only use it when at least one of the following conditions is met: (i). when BO has resulted in sampling at the design space boundary, which can be defined as: $min_{vec}(\boldsymbol{x}^*-\boldsymbol{B}_{\mathcal{X}_l})<\epsilon$\ \textbf{or}\ $min_{vec}(\boldsymbol{B}_{\mathcal{X}_u}-\boldsymbol{x}*)<\epsilon$, where $\boldsymbol{B}_{\mathcal{X}_l}, \boldsymbol{B}_{\mathcal{X}_u}$ represents the lowest and largest point coordinates that can define the design space, $\epsilon$ is a small non-negative threshold, $min_{vec}$ is the coordinate-wise minimum operator, (ii). when BO has resulted in duplicate sampling in the design space: $min||\boldsymbol{X}-\boldsymbol{x}^*||<\epsilon$ . Assuming the acquisition function has resulted in sampling $\boldsymbol{x}^*$, we propose to perform the AL optimization step within the bounded space $\mathcal{B}: [\boldsymbol{x}^* - \Delta\boldsymbol{x}, \boldsymbol{x}^* + \Delta\boldsymbol{x}]$, where $\Delta\boldsymbol{x}$ is a hyperparameter (illustrated in Fig. \ref{Fig: AL_acq}) that needs to be specified upfront. For bounded input uncertainty distributions, this can be intuitively specified as the distribution boundary; for the unbounded input uncertainty distribution like Gaussian distribution, a distance between the mean and 97.5 percentage of the marginal distribution can be chosen as $\Delta \boldsymbol{x}$. 

\subsection{Framework Outline}  

\begin{figure}[h]
  \begin{algorithm}[H]
  \SetAlgoLined
   \caption{Robust Multi-Objective Bayesian Optimization considering Input Uncertainty  (RMOBO-IU)}
     \label{Alg: rMOBO_alg}
   \textbf{Input}: max iter: $N_{iter}$, design space : $\mathcal{X}$ , training data $D =\{\boldsymbol{X}, \boldsymbol{Y}\}$, query pool: $Q = \{\}$, design space boundary stack: $\boldsymbol{B}_\mathcal{X}=\{\boldsymbol{B}_{\mathcal{X}_l}, \boldsymbol{B}_{\mathcal{X}_u}\}$, minimum distance threshold: $\epsilon$, $\Delta x$  ;\\
   \For {$i := 1$ to $N_{iter}$}
   {
      construct model based on $D$: 
      $ \mathcal{M}_{\boldsymbol{J}}: \{J_1\sim\mathcal{GP}'_1, ..., J_M\sim\mathcal{GP}'_M\}$ \\
      $\boldsymbol{X}_q^*$ = $arg \underset{\boldsymbol{x}\in \mathcal{X}}{max}\  \alpha(\boldsymbol{X}_q, \psi, \mathcal{M}_{\boldsymbol{J}})$ \\
      Augment query pool: $Q = \{Q \cup \boldsymbol{X}_q^*\}$\\
      initialize AL and BO pool: $X_k^{**} = \{\}$, $X_{q\setminus k}^* = \{\}$\\
      \For {$j := 1$ to $q$}
      {
            \eIf {$min||\boldsymbol{X}-\boldsymbol{x}_j^*||<\epsilon$ \   \textbf{or} \  $min_{vec}(\boldsymbol{x}_j^*-\boldsymbol{B}_{\mathcal{X}_l})<\epsilon$\ \textbf{or}\ $min_{vec}(\boldsymbol{B}_{\mathcal{X}_u}-\boldsymbol{x}_j^*)<\epsilon$}{
   $\boldsymbol{x}_j^{**} = arg \underset{\boldsymbol{x}\in [\boldsymbol{x}_j^{*} - \boldsymbol{\Delta_{\boldsymbol{x}}}, \boldsymbol{x}_j^{*} + \boldsymbol{\Delta_{\boldsymbol{x}}}]}{max}\  \alpha_{\text{AL}}(\boldsymbol{x}, \boldsymbol{x}_j^{*}, \mathcal{M}_{\boldsymbol{J}}, \mathcal{M}_{\boldsymbol{f}})$\\$\boldsymbol{X}_k^{**} = \boldsymbol{X}_k^{**}\cup \boldsymbol{x}_j^{**}$}{$\boldsymbol{X}_{q\setminus k}^*=\boldsymbol{X}_{q\setminus k}^*\cup \boldsymbol{x}_j^*$}
   Concatenate: $\boldsymbol{X}_{q}^{**} =  \boldsymbol{X}_k^{**} \cup \boldsymbol{X}_{q\setminus k}^{*}$
      }
   Query observations and augment training data: $D = \{D \cup \{ X_{q}^{**}, \boldsymbol{f}(X_{q}^{**})\} \}$
   }
   Concatenate optimal candidates : $\boldsymbol{X}_{cand} = \{\boldsymbol{x} \in \mathcal{X}: \boldsymbol{x} \in \boldsymbol{X} \cup Q\}$\\ 
  \textbf{Output} ranking on model inferred optimal candidates: $A_{rank}(\mathcal{M}_{\boldsymbol{J}}(\boldsymbol{X}_{cand}))$, robust model: $\mathcal{M}_{\boldsymbol{J}}$
  \end{algorithm}
\end{figure}

The complete RMOBO-IU approach is presented in Algorithm. \ref{Alg: rMOBO_alg}. The main paradigm is similar to a standard BO flow. Starting with a limited amount of data, the R-GP is constructed, and the Bayes risks are inferred. The first stage acquisition optimization (line 4) is conducted to search for the robust Pareto optimal points. Next, in the second phase, AL process (line 7-13) is utilized as needed to pursue better sampling candidates. Once the optimization has stopped, the Pareto front $\mathcal{F}_{\boldsymbol{J}}^*$ can be extracted based on the final models (out-of-sample) or on the sampled points (in-sample). We also note that this framework can be used for single objective robust BO if the objective number $M=1$, and the ranking operation in Eq. \ref{Eq: myopic_acq_2} is defined as $A_{rank}: = max(\cdot)$. 

\section{Numerical Investigation} \label{Sec:numerical_benchmark}

\begin{table}[hbpt]
\caption{Bi-objective benchmark function settings (see also appendix \ref{App: synthetic func}), where $t(\cdot), \mathcal{N}(\cdot), Tr\mathcal{N}(\cdot), \text{U}(\cdot)$ represents the student-t, normal, truncated normal and uniform distribution respectviely.}
\begin{center}
\begin{tabular}{llllll}
\hline 
\multicolumn{1}{l}{\bf Function}  & $\xi $ distribution & Input & Problem  & AL design \\
  &  & Dimension & Type (Fig, \ref{fig: different_pf}) &space  $\Delta \boldsymbol{x}$\\

\hline
VLMOP2 & $t(200, 0, 0.01^2)$& 2&C.1 & [0.0166, 0.0166]\\
SinLinForrester              & $\mathcal{N}(0, 0.05^2)$ & 1& C.2& [0.098, 0.098]\\

MDTP2 & $Tr\mathcal{N}([0, 0], [0.02^2, 0.04^2]$ &  2&\multirow{3}{*}{C.3} & [0.05, 0.05]\\
& \quad$[-0.05, -0.05], $ \\
& \quad [0.05, 0.05])&  \\
MDTP3    &U([-2e-2, -0.1], & 2&\multirow{2}{*}{C.4}& [0.02, 0.1]\\
& \quad\  [2e-2, 0.1]) \\
BraninGMM        & U([-0.2, -0.2], & 2& C.4 & [0.02, 0.02]\\
        &\quad\  [0.2, 0.2] \\
\hline\\
\end{tabular}
\end{center}
\label{Tab:benchmark}
\end{table}

There are relatively few benchmark functions in literature for robust multi-objective optimization. Therefore, we construct some new synthetic functions for benchmarking RMOBO-IU. They are listed in Table \ref{Tab:benchmark} and detailed in appendix \ref{App: synthetic func}, and used with various input uncertainty distributions. We note that these new synthetic functions cover all 4 cases that we have discussed in Section. \ref{seq: difference_pareto}. We employ the squared exponential kernel with a Maximum A Posterior (MAP) strategy, driven by the L-BFGS-B optimizer. We follow the same strategy of \cite{frohlich2020noisy} by specifying a log-normal prior on the lengthscales, and 2000 MC samples are used for approximating the kernel expectation. 

The code is implemented using the Trieste  library \citep{Berkeley_Trieste_2021}, and we test the RMOBO-IU framework using two popular acquisition functions for MOBO, i.e., EHVI and qEHVI. We start each benchmark with $5d$ initial data points uniformly generated in the design space, where $d$ is the problem dimensionality.  The experiments are conducted on a server with Intel(R) Xeon(R) CPUs E5-2640 v4 @ 2.40GHz, and each synthetic problem is repeated 30 times for robustness.

\begin{figure*}[!htbp]
\centering 
    \subfigure[VLMOP2 ]{ 
	\begin{minipage}{0.45\textwidth}
		\centering 
		\includegraphics[width=1.2\textwidth]{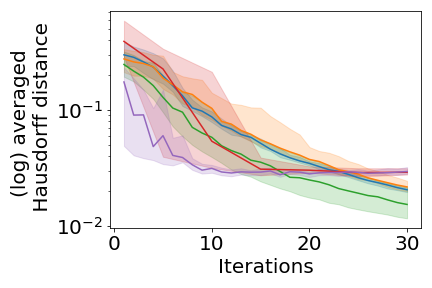} 
	\end{minipage}}\hfill
	\subfigure[SinLinForrester Function]{ 
	\begin{minipage}{0.45\textwidth}
		\centering 
		\includegraphics[width=1.2\textwidth]{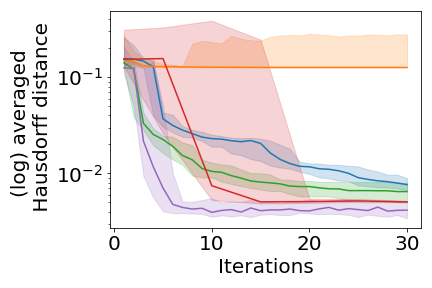} 
	\end{minipage}}\hfill
		\subfigure[MDTP2 Function]{ 
	\begin{minipage}{0.45\textwidth}
		\centering 
		\includegraphics[width=1.2\textwidth]{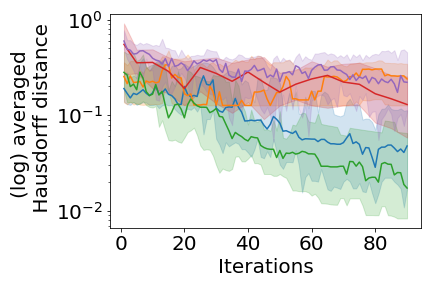} 
	\end{minipage}}\hfill
	\subfigure[MDTP3 Function]{ 
	\begin{minipage}{0.45\textwidth}
		\centering 
		\includegraphics[width=1.2\textwidth]{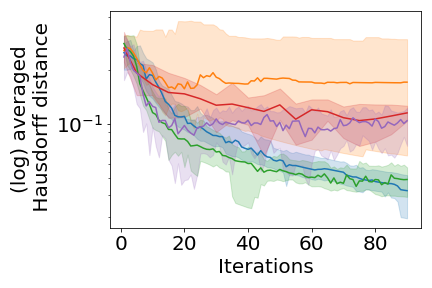} 
	\end{minipage}}\hfill
	\subfigure[BraninGMM]{ 
	\begin{minipage}{0.45\textwidth}
		\centering 
		\includegraphics[width=1.2\textwidth]{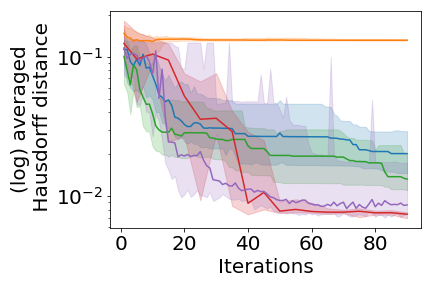} 
	\end{minipage}}\hfill
	\subfigure{ 
	\begin{minipage}{0.45\textwidth}
	    \vspace*{-2cm}
	    \hspace*{1cm}
		\centering 
		\includegraphics[width=0.6\textwidth]{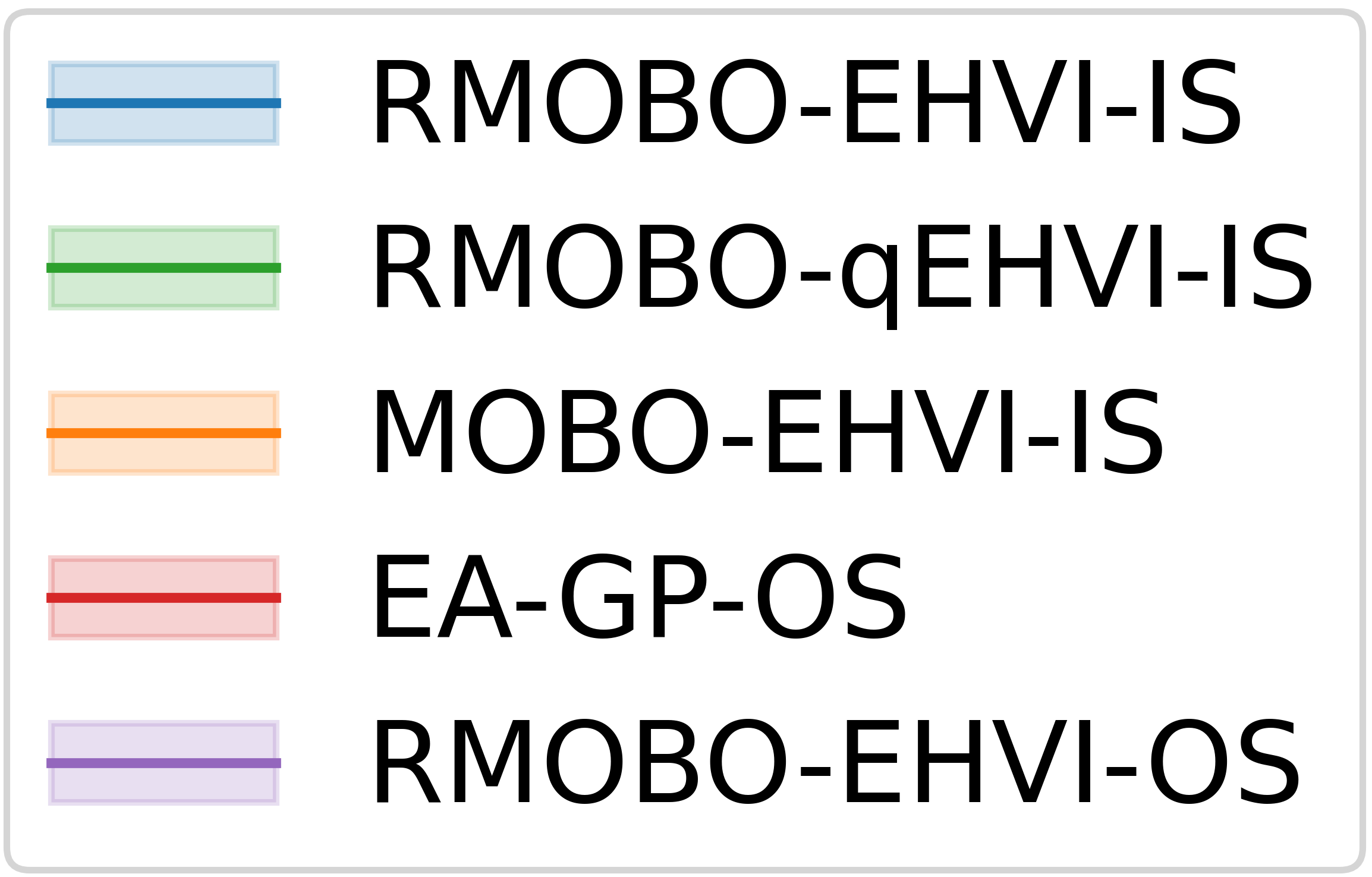} 
	\end{minipage}}\hfill
	\caption{Synthetic benchmark results for the AVD score with respect to the number of iterations. The median across 30 experiments is represented as a line and the 25/75th percentiles are reported as the shaded area.
} 
\label{fig:benchmark_res}	
\end{figure*}

\begin{figure*}[!htbp]
\centering 
    \subfigure[VLMOP2 \label{fig:IO}]{ 
	\begin{minipage}{0.45\textwidth}
		\centering 
		\includegraphics[width=1.2\textwidth]{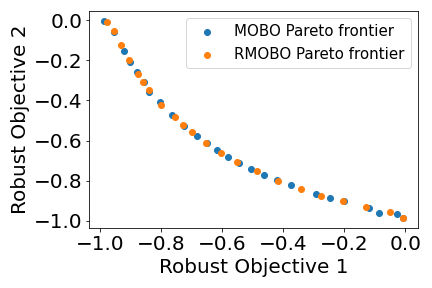} 
	\end{minipage}}\hfill
	\subfigure[SinLinForrester Function]{ 
	\begin{minipage}{0.45\textwidth}
		\centering 
		\includegraphics[width=1.2\textwidth]{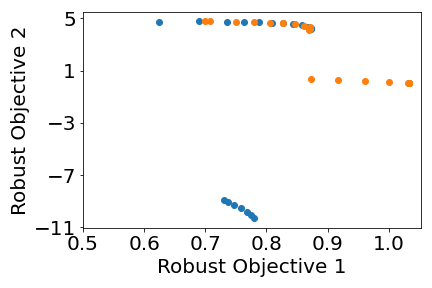} 
	\end{minipage}}\hfill
		\subfigure[MDTP2 Function]{ 
	\begin{minipage}{0.45\textwidth}
		\centering 
		\includegraphics[width=1.2\textwidth]{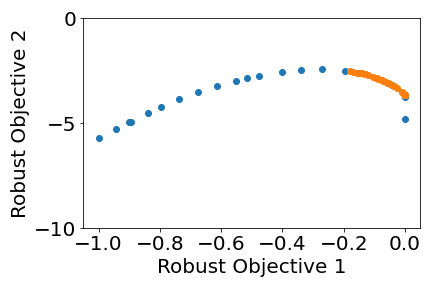} 
	\end{minipage}}\hfill
	\subfigure[MDTP3 Function]{ 
	\begin{minipage}{0.45\textwidth}
		\centering 
		\includegraphics[width=1.2\textwidth]{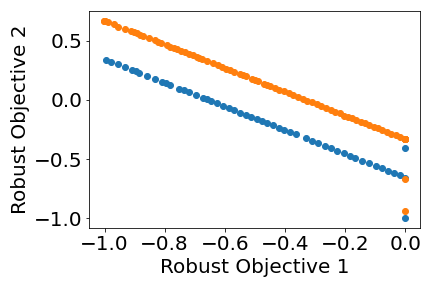} 
	\end{minipage}}\hfill
	\subfigure[BraninGMM]{ 
	\begin{minipage}{0.45\textwidth}
		\centering 
		\includegraphics[width=1.2\textwidth]{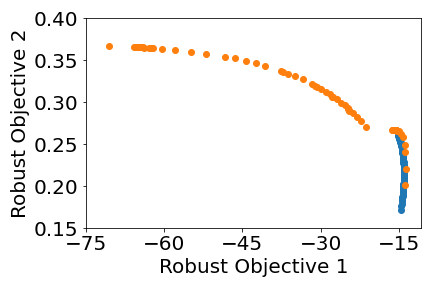} 
	\end{minipage}}\hfill
	\subfigure{ 
	\begin{minipage}{0.45\textwidth}
	    \vspace*{-2cm}
	    \hspace*{1cm}
		\centering 
		\includegraphics[width=0.8\textwidth]{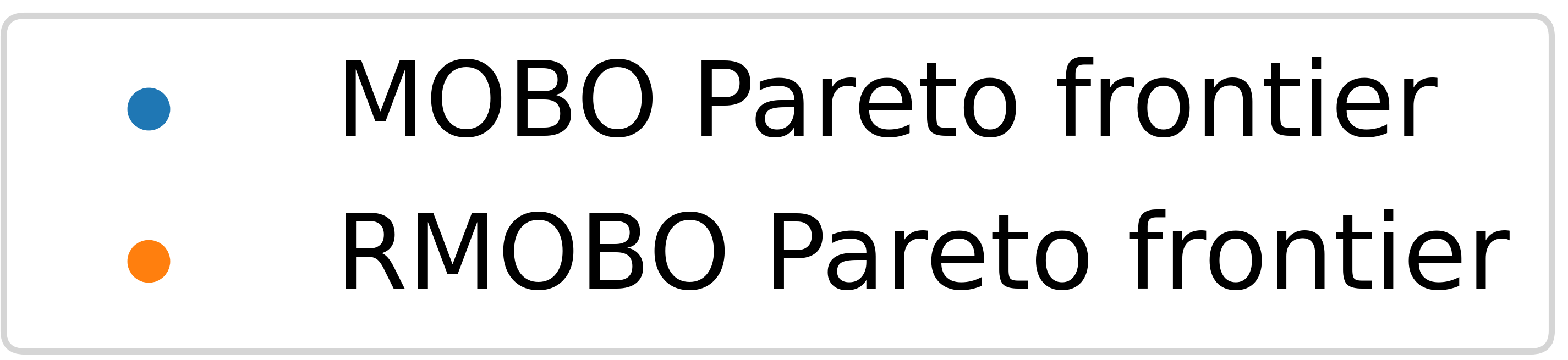} 
	\end{minipage}}\hfill
	\caption{Pareto front comparison of RMOBO and MOBO which doesn't inherently consider input uncertainty. With the consideration of input uncertainty, the Pareto front of MOBO (orange) is less optimal than RMOBO (blue) based on in-sample recommendations.
} 
\label{fig:benchmark_pareto_front}	
\end{figure*}

The performance is evaluated using the Averaged Hausdorff Distances (AVD)  based indicator (Eq. 45 of \citet{schutze2012using}) in the scaled objective space\footnote{When calculating the AVD metric, we scale the objective space to $[0, 1]^M$ based on the real Pareto front. This scaling aims to reduce bias from AVD if the magnitude between the objectives differ significantly.} as the performance metric with $p=2$. The reference Pareto frontier $\boldsymbol{F}^*$ is generated using an exhaustive NSGAII \citep{deb2002fast} search with population size 60. We compare RMOBO-IU (using an in-sample (IS) strategy) with standard MOBO, as well as a non-Bayesian MOO strategy. In the latter we use the NSGAII evolutionary algorithm (EA) based on a one-shot learned standard GPs as an Out-of-Sample (OS) strategy, which we refer to as the EA-GP-OS method. \footnote{The Bayes risk of the EA-GP-OS method is calculated using 2000 Monte Carlo samples on the GP posterior mean. For NSGAII we use a population size of 20 and 200 generations.}

The AVD's convergence histories of different acquisition functions\footnote{The qEHVI acquisition function has batch size $q = 2$.} and strategies are depicted in Fig. \ref{fig:benchmark_res}, and the final recommended Pareto fronts are shown in Fig. \ref{fig:benchmark_pareto_front}, more experiment detailes are sent to appendix \ref{App: exp_detail}. According to the results, it can be observed that for the VLMOP2 problem (case 1), its robust Pareto frontier is similar to its original Pareto frontier and that has led to similar convergence properties of the AVD measure. For the other cases, RMOBO-IU converges to the robust Pareto frontier while the non-robust MOO identifies of course non-robust solutions. We also note that in general a faster convergence speed can be observed by utilizing batch acquisition functions.

We also provide out-of-sample recommendations based on the R-GP for our RMOBO method to compare with EA-GP-OS, which we denote as RMOBO-EHVI-OS \footnote{We use the same NSGAII settings as used in EA-GP-OS.}. We note that RMOBO-EHVI-OS has in general an improved performance over EA-GP-OS, while the EA-GP-OS method is more robust for BraninGMM. Overall, while the out-of-sample strategies demonstrate better results than in-sample strategies on some benchmarks, their performance are not consistent across all problems. The worse performance can be shown especially on MDTP2 and MDTP3, where we deduce that if the problem is more difficult for an accurate surrogate model, the out-of-sample recommendation can have outliers of the Pareto frontier leading to a worse AVD score. Hence, we recommend to keep using in-sample strategies as a more robust choice. 

\section{Conclusion}
We presented RMOBO-IU: an approach for robust multi-objective optimization within the Bayesian optimization framework which considers input uncertainty. 

We optimize for Bayes risk, which is efficiently inferred using a robust Gaussian Process. The robust Gaussian Process is integrated in a two-stage Bayesian optimization process to search for the robust Pareto front. The effectiveness of the RMOBO-IU framework has been demonstrated on various new benchmark functions with promising results.

Future research will focus on several aspects: the SAA-MC-based kernel expectation still relies on sampling in the input space, which restricts its usage for a higher number of input dimensions. A more scalable approach is needed. Moreover, Bayesian versions of other robustness measures will also be investigated.

\begin{acknowledgements}
This research received funding from the Flemish Government (AI Research Program) and Chinese Scholarship Council under grant number 201906290032.
\end{acknowledgements}  

\noindent \small\textbf{Data availability Statement} The code for reproducing the experiments for the current study are available from the corresponding author on reasonable request.

\bibliography{reference.bib}   
\bibliographystyle{spbasic}      

\appendix

\section{Synthetic Functions} \label{App: synthetic func}
We provide a detailed description of the synthetic functions that we have utilized for numerical benchmarking, with a math formulation in Table. \ref{Tab:benchmark_detailed}. We note that the inverse of these synthetic functions is used to perform MOO for maximization.  

\noindent\textbf{VLMOP2} \citep{fonseca1995multiobjective} A bi-objective synthetic problem, where each objective function has only one global optima within the design space.

\noindent\textbf{MDTP2} A modified version of \cite{deb2005searching}'s test problem 2.

\noindent\textbf{SinLinForrester}  A bi-objective problem with SineLiner \citep{frohlich2020noisy} function and Forrester function \cite{forrester2008engineering}.  

\noindent\textbf{MDTP3} A modified version of \cite{deb2005searching}'s test problem 3. 

\noindent\textbf{BraninGMM} A bi-objective problem with Branin function \citep{picheny2013benchmark} and Gaussian Mixture Model \citep{frohlich2020noisy}, the input uncertainty is taken from \citep{beland2017bayesian}. 

\section{Experiment Details} \label{App: exp_detail}
In this section we demonstrate the experimental details, more specifically, we demonstrate the final query points of RMOBO-IU on the synthetic problem. The samples that RMOBO-IU investigated is illustrated in Fig. \ref{Fig: RMOBO_Xs}, once the AL process has been activated, the pending data is not the same as query data and the difference has been noted with the arrows. The one dimensional SinLinForrester function is omitted for its simplicity. It can be observed that RMOBO-IU is searching for locating at the robust Pareto frontier. Meanwhile, the AL optimization helps to alleviate the duplication and boundary issue in all the synthetic problems.

\begin{table*}[h]
\caption{Bi-objective benchmark functions settings}
\begin{center}
\begin{tabular}{lll}
\hline 
\multicolumn{1}{l}{\bf Function}  &\multicolumn{1}{c}{\bf Design Space } & Function Expression \\
\hline 
SinLinForrester & [0, 1]& $y_1 = sin(5\pi x^2) + 0.5x$\\
& & $y_2 = (6x-2)^2 sin(12x-4)$\\
\hline 
VLMOP2 & $[-2, 2]^2$ &$y_1 = 1 - \text{exp}(-\Sigma_{i=1}^2(x_i - \frac{1}{\sqrt{2}})^2)$ \\ && $y_2 = 1 - \text{exp}(-\Sigma_{i=1}^2(x_i + \frac{1}{\sqrt{2}})^2)$ \\
\hline 
MDTP2           &$[0, 1] \times [-1, 1]$ & $y_1 = x_1$\\
& & $y_2 = (1 - x_1^2)+(10 + x_2^2 - 10cos(4\pi x_2))\cdot$\\
& & $\quad\quad(\frac{1}{0.2 + x_1} + 10x_1^2)$\\
\hline 
MDTP3           &$[0, 1]^2$ & $y_1 = x_1$\\
& & $y_2 = 1 - 0.9\ e^{{(-\frac{x_2 - 0.8}{0.1}})^2} $\\
& & $\quad\quad- 1.3\ e^{(-\frac{x_2 - 0.3} { 0.03})^2}$ \\

\hline
BraninGMM  & $[0, 1]^2$& $y_1 = \frac{1}{51.95}[(x_2 - \frac{5.1x_1^2}{4\pi^2} + \frac{5x_1}{\pi}-6)^2 $\\
& & $\quad \quad  + (10 - \frac{10}{8\pi}cos(x_1)) - 44.81]$\\
& &  $y_2 = \sum_{j=1}^3p(z=j)p(x \vert z=j)$\\
& & where : \\
& & $p(z=1)=0.04\pi, x\vert z=1 \sim \mathcal{N}([0.2, 0.2], 0.2^2\delta)$\\
& & $p(z=2)=0.014\pi, x\vert z=2 \sim \mathcal{N}([0.8, 0.2], 0.1^2\delta)$\\
& & $p(z=3)=0.014\pi, x\vert z=3 \sim \mathcal{N}([0.5, 0.7], 0.1^2\delta)$\\
&& where $\delta$ represents Kronecker delta. \\
\hline\\
\end{tabular}
\end{center}
\label{Tab:benchmark_detailed}
\end{table*}

We illustrated part of the objective functions as well as their robust contour parts in Fig. \ref{Fig: objective comparison}, where the reference Pareto optimal points input are also illustrated in the figure.

\begin{figure*}[!htbp]
\centering 
\subfigure[MDTP2 Objective 2 original objective]{ 
	\begin{minipage}{0.45\textwidth}
		\centering 
		\includegraphics[width=1.2\textwidth]{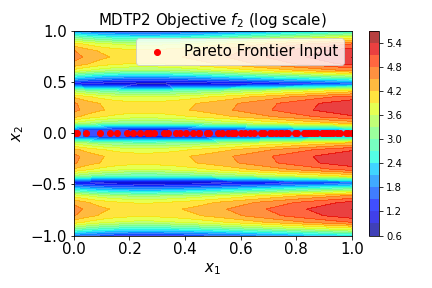} 
	\end{minipage}}\hfill
	\subfigure[MDTP2  Objective 2 robust objective]{ 
	\begin{minipage}{0.45\textwidth}
		\centering 
		\includegraphics[width=1.2\textwidth]{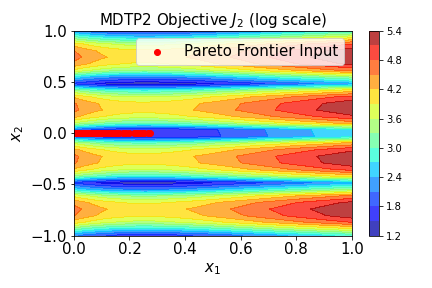} 
	\end{minipage}}\hfill
    \subfigure[MDTP3 Objective 2 original objective]{ 
	\begin{minipage}{0.45\textwidth}
		\centering 
		\includegraphics[width=1.2\textwidth]{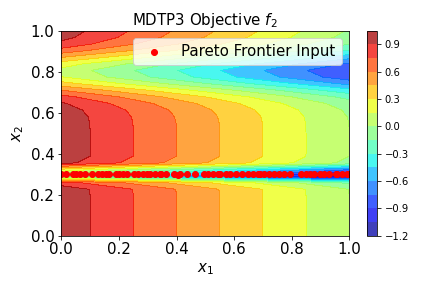} 
	\end{minipage}}\hfill
	\subfigure[MDTP3  Objective 2 robust objective]{ 
	\begin{minipage}{0.45\textwidth}
		\centering 
		\includegraphics[width=1.2\textwidth]{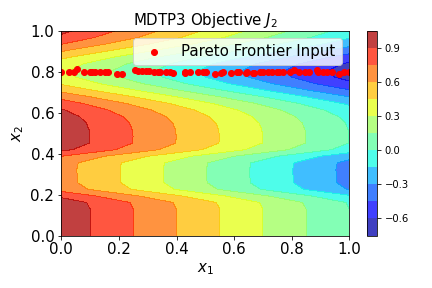} 
	\end{minipage}}\hfill
	    \subfigure[BraninGMM Objective 1 original objective]{ 
	\begin{minipage}{0.45\textwidth}
		\centering 
		\includegraphics[width=1.2\textwidth]{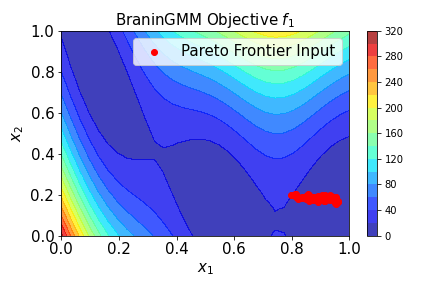} 
	\end{minipage}}\hfill
	\subfigure[BraninGMM  Objective 2 original objective]{ 
	\begin{minipage}{0.45\textwidth}
		\centering 
		\includegraphics[width=1.2\textwidth]{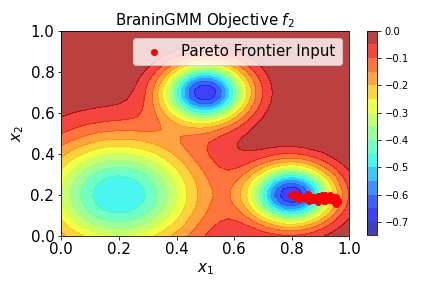} 
	\end{minipage}}\hfill
		    \subfigure[BraninGMM Objective 2 robust objective]{ 
	\begin{minipage}{0.45\textwidth}
		\centering 
		\includegraphics[width=1.2\textwidth]{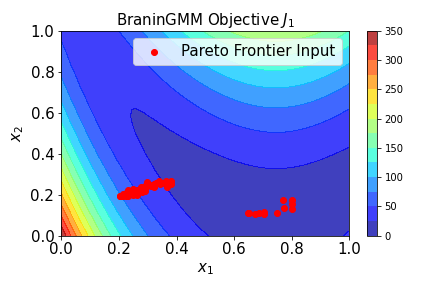} 
	\end{minipage}}\hfill
	\subfigure[BraninGMM  Objective 2 robust objective]{ 
	\begin{minipage}{0.45\textwidth}
		\centering 
		\includegraphics[width=1.2\textwidth]{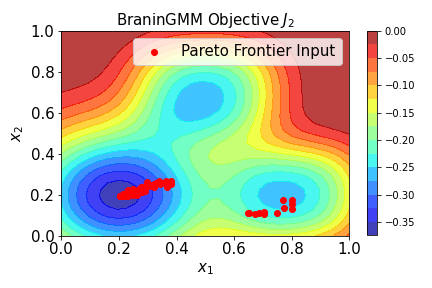} 
	\end{minipage}}\hfill
	\caption{Comparison of non-robust and robust objective functions, the inverse of which are used for maximization in numerical experiments. The corresponding Pareto frontier input is also illustrated in the figure, which has been obtained from NSGAII.
} 
\label{Fig: objective comparison}
\end{figure*}

\begin{figure*}[!hbpt]
\centering 
\subfigure[VLMOP2 RMOBO-IU Input Samples]{ 
	\begin{minipage}{0.45\textwidth}
		\centering 
		\includegraphics[width=1\textwidth]{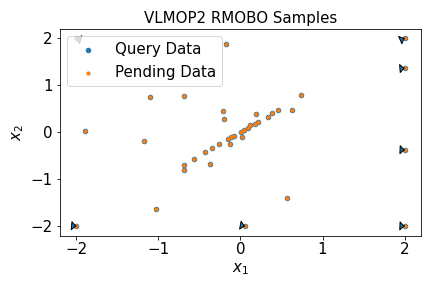} 
	\end{minipage}}\hfill
	\subfigure[MDTP2 RMOBO-IU Input Samples]{ 
	\begin{minipage}{0.45\textwidth}
		\centering 
		\includegraphics[width=1\textwidth]{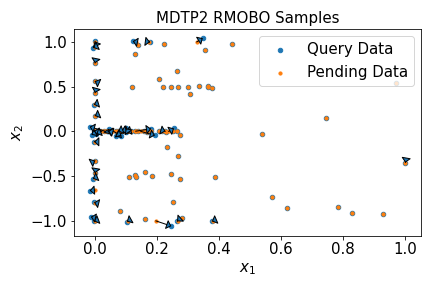} 
	\end{minipage}}\hfill
    \subfigure[MDTP3 RMOBO-IU Input Samples]{ 
	\begin{minipage}{0.45\textwidth}
		\centering 
		\includegraphics[width=1\textwidth]{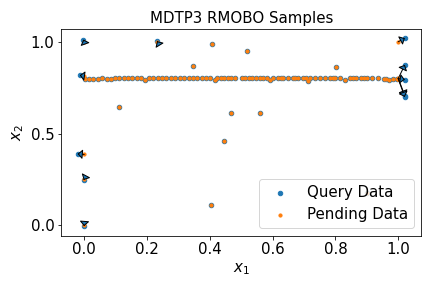} 
	\end{minipage}}\hfill
	\subfigure[BraninGMM RMOBO-IU Input Samples]{ 
	\begin{minipage}{0.45\textwidth}
		\centering 
		\includegraphics[width=1\textwidth]{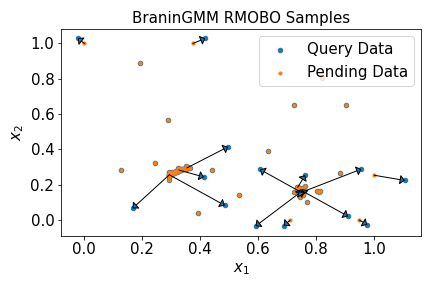} 
	\end{minipage}}\hfill

	\caption{Illustration of RMOBO-IU sample in input space (based on EHVI experiment).
} 
	\label{Fig: RMOBO_Xs}
\end{figure*}

%
%


\end{document}